%% file: main.tex
\pdfoutput=1
\documentclass{article}
\usepackage{natbib}
\usepackage[preprint]{icml2022}
\usepackage[utf8]{inputenc} % allow utf-8 input
\usepackage[T1]{fontenc}    % use 8-bit T1 fonts
\usepackage{hyperref}
\usepackage{url}            % simple URL typesetting
\usepackage{booktabs}       % professional-quality tables
\usepackage{amsmath,amsthm,amsfonts,amssymb,bm,stmaryrd}       % blackboard math symbols
\usepackage{nicefrac}       % compact symbols for 1/2, etc.
\usepackage{microtype}      % microtypography
\usepackage{enumitem}
\usepackage{support-caption}
\usepackage{subcaption}
\usepackage{mathtools}
\usepackage{nccmath}
\usepackage{multirow}

\usepackage[normalem]{ulem}
\useunder{\uline}{\ul}{}
\usepackage{scalerel}
\usepackage{xcolor}		% make links dark blue
\definecolor{darkblue}{rgb}{0, 0, 0.5}
\hypersetup{colorlinks=true, citecolor=darkblue, linkcolor=darkblue, urlcolor=darkblue}
\usepackage[notes=true, done=false, later=true, ]{dtrt}

\usepackage{booktabs}
\usepackage{dcolumn}
\newcolumntype{d}{D{.}{.}{-1}}
\input{macros}

\usepackage{graphicx}
\usepackage{titlesec}
\titlespacing*{\subsection}{0pt}{.15\baselineskip}{.15\baselineskip}
\titlespacing*{\subsubsection}{0pt}{.15\baselineskip}{.15\baselineskip}
\titlespacing*{\section}{0pt}{.2\baselineskip}{.2\baselineskip}
\renewcommand\cite{\citep}
\icmltitlerunning{Conformal Prediction Sets with Limited False Positives}

\begin{document}

\twocolumn[
\icmltitle{Conformal Prediction Sets with Limited False Positives}
%\icmltitle{Trading Coverage for Precision: \\ Conformal Prediction with Limited False Positives}



\icmlsetsymbol{equal}{*}

\begin{icmlauthorlist}
\icmlauthor{Adam Fisch}{mit}
\icmlauthor{Tal Schuster}{google}
\icmlauthor{Tommi Jaakkola}{mit}
\icmlauthor{Regina Barzilay}{mit}
\end{icmlauthorlist}

\icmlaffiliation{mit}{CSAIL, Massachusetts Institute of Technology.}
\icmlaffiliation{google}{Google Research}

\icmlcorrespondingauthor{Adam Fisch}{fisch@csail.mit.edu}

\icmlkeywords{Machine Learning, ICML}

\vskip 0.3in
]



\printAffiliationsAndNotice{}  % leave blank if no need to mention equal contribution

\input{ sections/abstract}
\input{ sections/introduction}
\input{ sections/related}
\input{ sections/background}
\input{ sections/conformal}
\input{ sections/setup}

\input{ sections/results}

\input{ sections/conclusion}

\bibliographystyle{plainnat}
\bibliography{main}

\clearpage
\appendix
\onecolumn
\counterwithin{figure}{section}
\counterwithin{table}{section}
\counterwithin{equation}{section}
\input{ sections/appendix/proofs}

\input{ sections/appendix/additional}
\input{ sections/appendix/tasks}
\input{sections/appendix/practical}

\end{document}

%% file: macros.tex
\newtheorem{theorem}{Theorem}
\numberwithin{theorem}{section}
\newtheorem{corollary}[theorem]{Corollary}
\newtheorem{lemma}[theorem]{Lemma}

\newtheorem{proposition}[theorem]{Proposition}

\theoremstyle{definition}
\newtheorem{definition}[theorem]{Definition}
\theoremstyle{definition}
\newtheorem{remark}[theorem]{Remark}

\DeclareMathOperator*{\argmax}{arg\,max}

\newcommand{\csetp}{\mathcal{C}}
\newcommand{\cinner}{\mathcal{C}_{\epsilon}^{\mathrm{inner}}}
\newcommand{\couter}{\mathcal{C}_{\epsilon}^{\mathrm{outer}}}
\newcommand{\cset}{\mathcal{C}_{k}}
\newcommand{\csetd}{\mathcal{C}_{k,\delta}}

\newcommand{\newpar}[1]{\textbf{{#1}.~}}

%% file: sections/abstract.tex
\begin{abstract}

We develop a new approach to multi-label conformal prediction in which we aim to output a precise set of promising 
prediction candidates with a bounded number of incorrect answers. Standard conformal prediction provides the ability to adapt to model uncertainty by constructing a calibrated candidate set in place of a single prediction, with guarantees that the set contains the correct answer with high probability. In order to obey this coverage property, however, conformal sets can become inundated with noisy candidates---which can render them unhelpful in practice. This is particularly relevant to practical applications where there is a limited budget, and the cost (monetary or otherwise) associated with false positives is non-negligible. We propose to trade coverage for a notion of precision by enforcing that the presence of incorrect candidates in the predicted conformal sets (i.e., the total number of false positives) is bounded according to a user-specified tolerance. Subject to this constraint, our algorithm then optimizes for a generalized notion of set coverage (i.e., the true positive rate) that allows for any number of true answers for a given query (including zero). We demonstrate the effectiveness of this approach across a number of classification tasks in natural language processing, computer vision, and computational chemistry.

\end{abstract}

%% file: sections/introduction.tex
\section{Introduction}
\label{sec:introduction}

For many classification problems, returning a set of plausible responses instead of a single prediction is a useful way of representing uncertainty~\cite{gammerman2007hedging, 10.1093/biomet/asu038, Romano2020ClassificationWV}. Conformal prediction~\cite{vovk2005algorithmic} is an increasingly popular method for creating confident {prediction sets} that provably contain the correct answer with high probability. Unfortunately, these guarantees do not come for free; in order to achieve proper coverage on difficult tasks, conformal prediction can often be unable to rule out an overwhelming number of candidates---making their prediction sets large and inefficient. This can make conformal predictors  unusable in settings in which the cost of returning false positive predictions is substantial.

\begin{figure*}[!t]
    \centering
    \small
    \includegraphics[width=\linewidth]{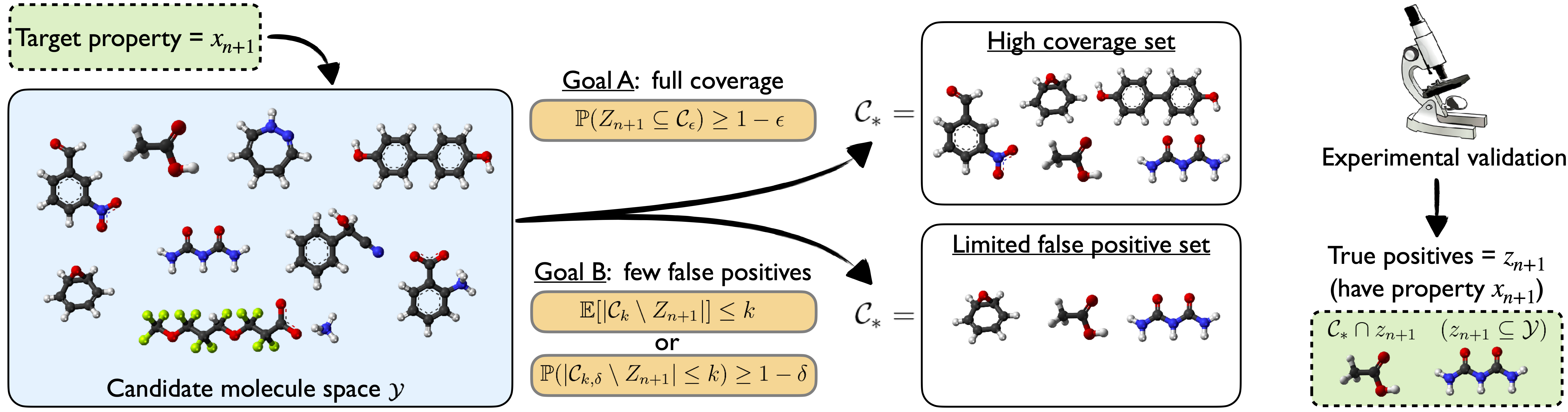}
    \vspace{-12pt}
    \caption{A demonstration of our approach to relaxing standard coverage guarantees (``Goal A'') in favor of rigorous limits on the total number of false positives included in the output $\csetd$ (``Goal B'').  In the illustrative case of in-silico screening for drug discovery, limiting false positives is critical when balancing a budget for experimental validation.}
    \label{fig:screeningexample}
   \vspace{-5pt}
\end{figure*}

As an example, consider in-silico screening for drug discovery (see Figure~\ref{fig:screeningexample}). In-silico screening uses computational tools to search over millions of molecular compounds to identify candidates with desired properties. Any identified candidates are then verified experimentally.  While it is often not necessary to return \emph{all} possible viable candidates (e.g., even identifying just one effective drug can suffice), it is important to respect budgetary constraints by avoiding false positive predictions. Too many false positives can quickly consume available resources (e.g., time, materials, funding, or other assets). This is especially relevant when a valid answer, in this case, an effective drug, might not even exist.

In this work, we develop an approach to creating confident prediction sets that trades off standard coverage guarantees for practical, provable constraints on the total \emph{number of false positives} (FP). In other words, we shift the focus of our conformal guarantees to be on limiting the number of incorrect answers in our outputs, with the understanding that we can potentially fail to recover some proportion of the true answers---i.e., we may obtain a lower \emph{true positive rate} (TPR), which we assume is acceptable for the application.

Concretely, we are interested in a set prediction setting where we have been given $n$ multi-label classification examples $(X_i, Z_i) \in \mathcal{X} \times 2^\mathcal{Y}$, $i = 1, \ldots n$ as calibration data, that have been drawn exchangeably from some  underlying distribution $P_{XZ}$. Under our assumptions, each observation $X_i$ can be associated with any number of correct labels (including zero, in the case of having no answer at all, or one, like standard classification). That is, the response variable $Z_i$ is a subset of the full label space $\mathcal{Y}$. For example, in the above in-silico screening task, $X_{i}$ would be the current property being screened for, $\mathcal{Y}$ the space of all molecular candidates that might have this property, and $Z_i \subseteq \mathcal{Y}$ the  set of molecules that do have it.
Let $X_{n+1} \in \mathcal{X}$ be a new exchangeable test example for which we would like to predict the set of correct labels, $Z_{n+1} \subseteq \mathcal{Y}$. Our goal is to construct a set predictor $\cset(X_{n+1})$ that maximizes recall of $Z_{n+1}$ (i.e., TPR), while limiting the expected number of false positives according to a user-defined tolerance $k \in \mathbb{R}_{>0}$:
\begin{equation}
\label{eq:kfd}
\begin{split}
    &\mathrm{maximize}~\mathbb{E}\left[ \frac{|\cset(X_{n+1}) \cap Z_{n+1}|}{\max(|Z_{n+1}|, 1)} \right] \\
    &\hspace{1cm}\text{s.t.}~~~\mathbb{E}\Big[ {|\cset(X_{n+1}) \setminus Z_{n+1}|} \Big] \leq k.
\end{split}
\end{equation}
As an alternative to bounding the \emph{expected number} of false positives, we can also seek a predictor $\csetd$ that has more direct control of the \emph{probability of exceeding} $k$ false positives:
\begin{equation}
\label{eq:kdeltafd}
\begin{split}
    &\mathrm{maximize}~\mathbb{E}\left[ \frac{|\csetd(X_{n+1}) \cap Z_{n+1}|}{\max(|Z_{n+1}|, 1)} \right] \\
    &\hspace{1cm}\text{s.t.}~~~\mathbb{P}\Big( {|\csetd(X_{n+1}) \setminus Z_{n+1}|} \leq k \Big) \geq 1 - \delta,
\end{split}
\end{equation}
where $\delta \in (0,1)$ is another user-defined tolerance level. Both constructions define different, but  useful, operating conditions; the first is more straightforward (e.g., for the general practitioner), while the second offers a finer, two-parameter level of control. Note that both constraints are marginal over the choice of calibration and test data. 

In order to achieve the desired levels of false positive control, we present an approach that is based on \emph{set classification}, combined with conformal calibration techniques~\cite{shafer2008tutorial, Papadopoulos08, ALVARSSON202142}. Specifically, we use a set nonconformity measure $\mathcal{F} \colon \mathcal{X} \times 2^\mathcal{Y} \rightarrow \mathbb{R}$ to score candidate output sets, $\mathcal{S} \in 2^\mathcal{Y}$, for a given input $x \in \mathcal{X}$.  Intuitively, a \emph{high} nonconformity score (e.g., loss) should reflect the confidence that the candidate set might contain a \emph{high} number of false positives, and vice-versa. We learn this function from separate multi-label classification training data. As enumerating and scoring all possible candidate sets is combinatorially hard, we instead adopt the nested conformal prediction strategy of \citet{Gupta2019NestedCP}, where we greedily construct prediction sets  using a best-first strategy that adds top-ranked individual labels to a growing, nested output set $\mathcal{S}$. We stop when its nonconformity score, $\mathcal{F}(x, \mathcal{S})$,  exceeds a calibrated threshold---that we find based on our desired false positive constraints. This greedy approach both allows us to scale to larger label spaces $\mathcal{Y}$ (i.e., where there are many candidate labels that choose from when composing the prediction set), and to leverage powerful theory for calibrating expectations of monotonic losses for nested set predictors~\cite{Gupta2019NestedCP, bates-rcps, arxiv-note}.

In summary, our main contributions are as follows:
\begin{itemize}[leftmargin=*, noitemsep]
    \vspace{-5pt}
    \item A theoretical adaptation of conformal prediction that provides rigorous false positive control instead of coverage;\vspace{3pt}
    \item A simple and effective strategy for constructing valid output sets with empirically high true positive rates;\vspace{3pt}
    \item A demonstration of the practical utility of our framework across a range of diverse classification tasks.
\end{itemize}

%% file: sections/related.tex
\section{Related work}
\label{sec:related}

\newpar{Uncertainty estimation}  A large body of work in estimating model uncertainty focuses on {calibrating} model-based conditional probabilities,  $p_\theta(\hat{y}_{n+1}|x_{n+1})$, such that the accuracy, $y_{n+1} = \hat{y}_{n+1}$, is equal to the estimated probability~\cite[][]{Brier1950VERIFICATIONOF,VerificationofProbabilisticPredictionsABriefReview, niculescu2005predicting, pmlr-v80-kuleshov18a, kumar-verified-2019, pmlr-v89-vaicenavicius19a}. In theory, these estimates could be used to create prediction sets with few false positives, but they are not always  accurate~\cite{pmlr-v70-guo17a, ashukha2020pitfalls,  hirschfeld2020uncertainty}. In a similar vein, Bayesian formalisms  quantify uncertainty via computing the posterior predictive distribution over model parameters~\cite[][]{neal1996bayesian, graves2011vi, hernandez2015bnn, gal2016dropout}. However, the quality of these methods can vary depending on the suitability of the presumed prior and on  approximation error.

\newpar{Conformal prediction} As introduced in \S\ref{sec:introduction}, conformal prediction~\cite{vovk2005algorithmic} provides a finite-sample, distribution-free method for obtaining prediction sets $\mathcal{C}$ with guarantees on the event $\mathbf{1}\{Y_{n+1} \in \mathcal{C}(X_{n+1})\}$. Most efforts in CP focus on improving the predictive efficiency, $\mathbb{E}[|\mathcal{C}(X_{n+1})|]$, of the conformal sets~\cite[][]{vovk2016criiteria, sadinle-least-2019, Romano2020ClassificationWV, angelopoulos2021sets, fisch2021admission,fisch2021fewshot, hoff-bayes-2021}. As  coverage is guaranteed by design, improving efficiency will naturally lead to more precise sets with fewer false positives---but not to a specifiable level. \citet{cauchois2020knowing} develop a conformal approach to multi-label classification that can guarantee that the prediction set  {only} contains true labels (i.e., FP = 0), but does not offer fine-grained control. Most relevant to our work, \citet{bates-rcps} develop a flexible framework for controlling the {risk}, $\mathbb{E}[\mathcal{L}(Y, \mathcal{T}(X))]$, of a set-valued predictor $\mathcal{T}$ with an arbitrary loss function $\mathcal{L}$---as long the loss respects a monotonic \emph{nesting} property, $\mathcal{S} \subset \mathcal{S}' \Rightarrow \mathcal{L}(\mathcal{S}) \geq \mathcal{L}(\mathcal{S}')$, for any two prediction sets $\mathcal{S}$ and $\mathcal{S}'$. The calibration strategy we use here for marginal expectations is based on an  extension  in \citet{arxiv-note}.
Recently, \citet{anastasios-learning-2021} proposed methods to rigorously control non-monotonic losses, including the related {false discovery rate} (FDR), which normalizes the number of false positives over the size of the prediction set. However, as most of our target applications have relatively few true positives, FDR control can  lead to many empty predictions (making controlling total false positives a more natural fit for this work, see Appendix~\ref{app:practical}). Finally, though we focus on conformal approaches, our methods are tightly connected to the broader literature surrounding distribution-free calibration~\cite[][]{NIPS2003_10c66082, NIPS2015_a9a1d531, Vovk2014VennAbersP, Gupta2019NestedCP, NEURIPS2020_26d88423, barber-2020}.

\newpar{Multiple testing} Controlling the number of false positives/discoveries over a collection of hypothesis tests is  well-studied~\cite[][]{dunn-multiple-1961, Benjamini1995, lehmann-fwer-2005, romano-control-2007}. In fact, the objectives expressed in Eqs.~\eqref{eq:kfd} and \eqref{eq:kdeltafd} are established concepts in statistics---i.e., PFER, the per-family error rate, and $k$-FWER, the familywise error rate \cite{spjotvoll-1972, romano-control-2007}.  Recently, FDR control has also been studied for outlier detection in a conformal inference setting~\cite{bates2021testing}. Classic approaches operate over p-values for each hypothesis test that have specific dependency structures (e.g., independent or positively dependent), or otherwise use more conservative corrections. Though similar, our multi-label setting is slightly different from standard multiple testing in that there is both (1) an unknown dependency structure between candidate labels for the same query, but also (2) an extra layer of exchangeability over the $n + 1$ queries. Our approach is able to ignore (1) by leveraging (2) within a conformal calibration framework.

\newpar{Selective classification} In {selective classification}~\cite{elyaniv2010selective},  models can abstain from answering. In particular, \citet{geifman2017selective} propose a strategy for finding classifiers with specific selective $0/1$ risks (i.e., the expected accuracy over \emph{answered} examples). In our setting, this is analogous to controlling  false positives using $k \approx 0$. If uncertain, the model would have to ``abstain'' by outputting an empty set. Our framework generalizes this behavior to other types of constraints for any positive $k$.

%% file: sections/background.tex
\section{Background}
We begin with a brief review of conformal prediction~\cite[see][]{shafer2008tutorial}.
Here, and in the rest of the paper, upper-case letters ($X$) denote random variables; lower-case letters ($x$) denote constants, and script letters ($\mathcal{X}$) denote sets, unless otherwise specified.
Proofs are in Appendix~\ref{app:proofs}.

Given a new example $x$, for every candidate label $y \in \mathcal{Y}$ standard conformal classification  (where there is one correct output) either accepts or rejects the null hypothesis that the pairing $(x, y)$ is correct. The test statistic for this test is a \emph{nonconformity measure}, $\mathcal{M}\left((x, y), \mathcal{D}\right)$, where $\mathcal{D}$ is a dataset of exchangeable, labeled examples (as is $(x, y_{\mathrm{true}})$). Informally, a lower value of $\mathcal{M}$ reflects that point $(x, y)$ ``conforms'' to $\mathcal{D}$, whereas a higher value of $\mathcal{M}$ reflects that $(x, y)$ is atypical relative to $\mathcal{D}$. A practical choice for $\mathcal{M}$ could be a model-based loss, e.g.,  $-\log p_\theta(y | x)$, where $\theta$ is a model fit to $\mathcal{D}$. For conformal prediction to work, it is important to ensure that $\mathcal{M}$ preserves the exchangeability over $\mathcal{D} \cup (x, y_{\mathrm{true}})$. One such way is to learn $\mathcal{M}$ on separate data. Split conformal prediction~\cite{Papadopoulos08} uses a proper training set $\mathcal{D}_{\mathrm{train}}$ to learn a fixed  $\mathcal{M}$ that is not modified during calibration or prediction. This trivially preserves exchangeability of the calibration and test points, and is a computationally efficient strategy (which we follow).

To construct a {prediction} set for the new test point $x$, the conformal classifier outputs all $y$ for which the null hypothesis (that pairing $(x, y)$ is correct) is not rejected. This is achieved by comparing the scores of the test candidate pairs to the scores computed over  $n$ calibration examples.
\begin{theorem}[Split CP, \citet{vovk2005algorithmic, Papadopoulos08}]
\label{thm:conformalprediction}
Assume that examples $(X_i, Y_i) \in \mathcal{X} \times \mathcal{Y}$, $i=1,\ldots, n+1$ are exchangeable. For a fixed nonconformity measure $\mathcal{M}$, let random variable $V_{i} = \mathcal{M}(X_i, Y_i)$ be the nonconformity score of $(X_i, Y_i)$. For $\epsilon \in (0, 1)$, define the prediction (based on the first $n$ examples) at  $x \in \mathcal{X}$ as
\begin{align}
\label{eq:csetconstruction}
&\mathcal{C}_\epsilon(x) := \\ &\quad\big \{ y \in \mathcal{Y} \colon 
    \mathcal{M}(x, y) \leq \mathrm{Quantile}(1 - \epsilon;\, V_{1:n} \cup \{\infty\}) \big\}.\nonumber
\end{align}
Then $\mathcal{C}_\epsilon(X_{n+1})$ satisfies $\mathbb{P}(Y_{n+1} \in \mathcal{C}_{\epsilon}(X_{n+1})) \geq 1 - \epsilon$.
\end{theorem}

\begin{remark}
\citet{cauchois2020knowing} extend the single label conformal prediction formulation to the multi-label case, where $Z_{n+1} \subseteq \mathcal{Y}$, by predicting two sets $\cinner, \couter \subseteq \mathcal{Y}$ that fully sandwich $Z_{n+1}$, i.e., they guarantee $\mathbb{P}(\cinner(X_{n+1}) \subseteq Z_{n+1} \subseteq \couter(X_{n+1})) \geq 1 - \epsilon$.
\end{remark}

The motivation for our work is evident from Eq.~\eqref{eq:csetconstruction}: if we are unable to reject most candidates based on their nonconformity scores, then  $\mathcal{C}_\epsilon$ can contain many false positives.

%% file: sections/conformal.tex
\section{\mbox{Set predictions with limited false positives}}

We now introduce our strategy for limiting the number of false positives that are contained in our output sets.  To briefly remind the reader of our setting,  we assume that we have been given $n$ exchangeable multi-label classification examples, $(X_i, Z_i) \in \mathcal{X} \times 2^\mathcal{Y}$, $i = 1, \ldots n$ as calibration data, that are drawn from a distribution $P_{XZ}$. We follow split conformal prediction, and assume that any training data used is distinct from this calibration data. The response $Z_i$ is treated as a generalized set of correct labels for input $X_i$, and is a subset of $\mathcal{Y}$. For example, in the in-silico screening task from \S\ref{sec:introduction}, $X_{i}$ is the current target property being screened for, $\mathcal{Y}$ is the space of all molecular candidates, and $Z_i \subseteq \mathcal{Y}$ is the set of molecules that have that property.

For a prediction set $\csetp(x) \subseteq \mathcal{Y}$ evaluated at a point $x \in \mathcal{X}$ with label set $z \subseteq \mathcal{Y}$, we define the \emph{true positive proportion} (TPP) as the ratio of correct labels that are recovered:
\begin{equation}
\label{eq:TPR}
    \mathrm{TPP}(z, \csetp(x)) := \frac{|\csetp(x) \cap z|}{\max(|z|, 1)} % \quad\quad \text{(Note that $\mathrm{TPR} := \mathbb{E}[\mathrm{TPP}]$)},
\end{equation}
(note that $\mathrm{TPR} := \mathbb{E}[\mathrm{TPP}]$), and the number of \emph{false positives} (FP) as the total count of incorrect labels in $\csetp(x)$:
\begin{equation}
    \mathrm{FP}(z, \csetp(x)) := |\csetp(x) \setminus z|.
\end{equation}
Our goal, as stated in \S\ref{sec:introduction}, is to maximize the expected TPP, while constraining  the FP in either of two ways:
\begin{definition}[$k$-FP validity] A conformal classifier producing random test prediction
$\cset(X_{n+1})$ is $k$-FP valid if it satisfies $\mathbb{E}[\mathrm{FP}(Z_{n+1}, \cset(X_{n+1}))] \leq k$.
\end{definition}
\begin{definition}[$(k, \delta)$-FP validity] A conformal classifier producing random test prediction $\csetd(X_{n+1})$ is  $(k, \delta)$-FP valid if it satisfies $\mathbb{P}(\mathrm{FP}(Z_{n+1}, \csetd(X_{n+1})) \leq k) \geq 1 - \delta$.
\end{definition}

\subsection{An oracle set predictor}
To motivate our approach, imagine an \emph{oracle} with access to $P_{Z\mid X}$, the conditional distribution of the multi-label set $Z$ given the input $X$. Given this information, for any input $x \in \mathcal{X}$ and candidate set $\mathcal{S} \in 2^{\mathcal{Y}}$, in theory such an oracle would be able to exactly calculate both the expectation and the conditional distribution of the number of false (and true) positives in $\mathcal{S}$ given $x$. In order to maximize the TPR while meeting $k$-FP and $(k, \delta)$-FP validity, it could then yield:
\begin{align}
&\mathcal{C}_{k}^{\mathrm{oracle}}(x) := \label{eq:oraclek}\\
&~\underset{\mathcal{S} \in 2^\mathcal{Y}}{\argmax}~\big\{\mathbb{E}[{\mathrm{TPP}}(Z, \mathcal{S}) \mid x] \colon \mathbb{E}[{\mathrm{FP}}(Z, \mathcal{S}) \mid x] \leq k \big\}\nonumber \\[5pt]
&\mathcal{C}_{k, \delta}^{\mathrm{oracle}}(x) := \label{eq:oraclekd}\\
&~\underset{\mathcal{S} \in 2^\mathcal{Y}}{\argmax}~\big\{\mathbb{E}[{\mathrm{TPP}}(Z, \mathcal{S}) \mid x] \colon \mathbb{P}({\mathrm{FP}}(Z, \mathcal{S}) \mid x] > k) < \delta \big\}\nonumber
\end{align}
where ties are settled by smaller set size. Of course, computing this oracle is not possible, as $P_{Z \mid X}$ is unknown. Furthermore, enumerating all sets $\mathcal{S} \in 2^\mathcal{Y}$ is infeasible for large $\mathcal{Y}$. Instead, in the following sections we develop a practical approach for roughly approximating the oracle's behavior with three main components: 
\begin{enumerate}[leftmargin=*, noitemsep]
    \item A \textbf{set function} $\mathcal{F} \colon \mathcal{X} \times \smash{2^\mathcal{Y}} \rightarrow \mathbb{R}$ that directly generates a score for a candidate {set} $\mathcal{S}$ given $x$ that is predictive of either $\mathbb{E}[{\mathrm{FP}}(Z, \mathcal{S}) \mid x]$ or $\mathbb{P}({\mathrm{FP}}(Z, \mathcal{S}) \mid x] > k)$.\vspace{5pt}
    \item A \textbf{calibrated search strategy} for exploring a {tractable} number of candidate sets, and identifying {valid} sets satisfying our constraints using predictions from $\mathcal{F}$;\vspace{5pt}
    \item A \textbf{selection policy} for picking a final output set. % from among those that are identified to be valid.
\end{enumerate}
Wherever possible, our proposed method will try to  balance simplicity and efficiency with effectiveness. Theoretically, however,  the framework  it follows is model-agnostic.
\input{sections/appendix/fp_control}
\subsection{Scoring candidate sets with set functions}
\label{sec:deepsets}

We choose to model $\mathcal{F}$ using $\mathrm{DeepSets}$~\cite{zaheer2017sets}. $\mathrm{DeepSets}$ is a popular method which is known to be a universal approximator for continuous set functions, which makes it a natural choice for our purpose.
%\footnote{Again, however, we emphasize that the set model $\mathcal{F}$ can be arbitrary---and we also test simpler variants in \S\ref{sec:baselines}.}
Let $\{\phi(x, y_1), \ldots, \phi(x, y_{|\mathcal{S}|})\}$ featurize a candidate set $\mathcal{S} \subseteq \mathcal{Y}$, where $\phi(x, y_c) \in \mathbb{R}^d$ is a function of $(x, y_c)$, for $y_c \in \mathcal{S}$. 
In practice, we find that taking $\phi(x, y_c)$ to be an estimate of $p_\theta(y_c \in Z \mid x)$, the marginal likelihood of $y_c$ being a correct label, performs  well and is simple to implement. These  one-dimensional prediction scores can be provided by any base model.\footnote{This is comparable to the $1$-$d$ features used by Platt scaling.} For example, in our in-silico screening task, we define $\phi$ using a directed MPNN~\cite{chemprop} that independently predicts the probability of an individual molecule having the properties targeted by the screen, or not. Given $\phi$, the $\mathrm{DeepSets}$ model is  defined by
\vspace{2pt}
\begin{equation}
\label{eq:deepsets}
    \Psi(x, \mathcal{S}) := \mathrm{softmax} \Big( \mathrm{dec} \Big(\sum_{y_c \in \mathcal{S}} \mathrm{enc}(\phi(x, y_c))\Big) \Big),
\end{equation}
where $\mathrm{enc}(\cdot)$ and $\mathrm{dec}(\cdot)$ are neural encoder and decoder models, and $\mathrm{softmax}(\cdot)$ is taken over the range of possible  false positives, $\{0, \ldots, |\mathcal{S}|\}$. $\Psi$ is trained to predict the total number of false positives in $\mathcal{S}$ via cross entropy, using labeled sets sampled from held-out training data, separate from the split used to learn $p_\theta$ (used for $\phi$). We then compute $\mathcal{F}_k$ and $\mathcal{F}_{k, \delta}$ (for either $k$-FP or $(k, \delta)$-FP validity) as
\begin{align}
    \mathcal{F}_k(x, \mathcal{S}) &:= \sum_{\eta=0}^{|\mathcal{S}|}\eta \cdot \Psi(x, \mathcal{S})_\eta \label{eq:fk} \\
    \mathcal{F}_{k, \delta}(x, \mathcal{S}) &:= 1 - \hspace{-7pt}\sum_{\eta=0}^{\min(k, |\mathcal{S}|)}\hspace{-7pt}\Psi(x, \mathcal{S})_\eta, \label{eq:fkd}
\end{align}
where $\Psi(x, \mathcal{S})_\eta$ denotes the $\eta$-th index of the $\mathrm{softmax}$ (i.e., the estimated probability that FP = $\eta$). In the next sections, we will now only refer to $\mathcal{F}$ as a general function. 

\subsection{Searching for valid candidate sets}
\label{sec:calibration}
Although our set predictor $\mathcal{F}$ is trained to model either the expected FP or its CDF, it is not necessarily accurate. If $\mathcal{F}$ were simply substituted into Eq.~\eqref{eq:oraclek} or Eq.~\eqref{eq:oraclekd}, it may not produce valid set predictions. To account for this mismatch, we must carefully calibrate a threshold for accepting  candidate sets based on $\mathcal{F}$. At the same time, we also must efficiently search the combinatorial space of candidate sets.

To efficiently calibrate our predictor, we cast our approach into a form of \emph{nested} conformal prediction~\cite{Gupta2019NestedCP}. First, we greedily identify a sequence of {nested}  candidate sets, $\varnothing \subset \mathcal{S}_1 \subset \mathcal{S}_2 \subset \ldots \subset \mathcal{S}_j$, by ranking individual labels $y_c \in \mathcal{Y}$ according to some auxiliary model,  and including them one by one into the growing output set $\mathcal{S}_{j+1}$.
Notice that, by construction, the number of false positives contained in $\mathcal{S}_j$ is non-decreasing in index $j$, i.e.,
\begin{equation}
    j \leq j' \Longrightarrow \mathrm{FP}(z, \mathcal{S}_j) \leq \mathrm{FP}(z, \mathcal{S}_{j'}).
\end{equation}
In practice, we find that ranking individual labels by their estimated marginal likelihoods of being true positives, i.e., $p_{\theta}(y_c \in Z \mid x)$---the same model used in \S\ref{sec:deepsets}---performs well and avoids the overhead of training an additional scoring model. Importantly, for further efficiency (elaborated on in Remark~\ref{rm:large}) we only consider sets up to a maximum size $B \leq |\mathcal{Y}|$, where $B$ is a hyper-parameter that we can set. 

Next, we compute a set nonconformity score $v_j$ (assumed to be finite) for each candidate set $\mathcal{S}_j$ using $\mathcal{F}$, where
\begin{equation}
\label{eq:vscore}
\vspace{-2pt}
    v_j := \mathcal{F}(x, \mathcal{S}_j).
\end{equation}
Finally, we define the worst-case number of false positives over all nested candidate sets $\mathcal{S}_{1:B}$ having nonconformity scores less than $t$ (given the input $x$ with label set $z$) as
\begin{equation}
\label{eq:fpmax}
   \mathrm{FP_{max}}(x, z, t) := \max\big\{\mathrm{FP}(z, \mathcal{S}_j) \colon v_j < t\big\}. 
\end{equation}
If this set is empty, then $\mathrm{FP_{max}}$ is $0$. Due to our nested construction, this is also simply the number of false positives contained in the largest set $\mathcal{S}_j$ satisfying $v_j < t$.  It is simple to show that $\mathrm{FP_{max}}$ is non-decreasing in $t$, as stated below.
\begin{lemma}[Monotonicity]
\label{lem:fpmax}
For sets $\mathcal{S}_j$ and scores $v_j$ and $\mathrm{FP_{max}}(x, z, t)$ as defined in Eqs.~\eqref{eq:vscore} and \eqref{eq:fpmax}, respectively,
\vspace{-5pt}
\begin{equation}
t \leq t' \Longrightarrow \mathrm{FP_{max}}(x, z, t) \leq \mathrm{FP_{max}}(x, z, t').
\end{equation}
\end{lemma}
Using this key property, we can find a \emph{maximal} threshold $t$ to use as a ``cutoff point'' for the sequence of nested candidate sets such that $\mathrm{FP_{max}}$ is controlled, as formalized next.
\begin{theorem}[FP-CP]
\label{thm:calibration}
Assume that examples $(X_i, Z_i) \in \mathcal{X} \times 2^\mathcal{Y}$, $i=1,\ldots, n+1$ are exchangeable. For each example $i$, let $S_{i,j}$, $j=1,\ldots, B$ (where $B \leq |\mathcal{Y}|$ is a finite hyper-parameter) be candidate sets, where finite random variable $V_{i,j} = \mathcal{F}(X_i, \mathcal{S}_{i,j})$ is a set nonconformity score. For tolerances $k \in \mathbb{R}_{> 0}$ and $\delta \in (0, 1)$ define the random variables $T_k$ and $T_{k, \delta}$  (based on the first $n$ examples) as
\vspace{-5pt}
\begin{align}
    &T_k := \label{eq:caltk} \\ &~~\sup \Big\{t \in \mathbb{R} \colon \frac{B + \sum_{i=1}^{n}\mathrm{FP_{max}}(X_i, Z_i, t)}{n+1} \leq k\Big\}\quad\text{and}\nonumber\\
    &T_{k,\delta} := \label{eq:caltkd}\\ &~~\sup \Big\{t \in \mathbb{R} \colon \frac{\sum_{i=1}^{n} \mathbf{1}\{\mathrm{FP_{max}}(X_i, Z_i, t) \leq k\}}{n+1} \geq 1 - \delta \Big\},\nonumber
\end{align}
where $\mathrm{FP_{max}}$ is as defined in Eq.~\eqref{eq:fpmax}. Then  we have that
\begin{align}
    &\mathbb{E}\Big[\mathrm{FP}_{\mathrm{max}}(X_{n+1}, Z_{n+1}, T_k)\Big] \leq k,\quad \text{and} \\
 &\mathbb{P}\Big(\mathrm{FP}_{\mathrm{max}}(X_{n+1}, Z_{n+1}, T_{k, \delta}) \leq k \Big) \geq 1 - \delta.
\end{align}
\end{theorem}
\begin{remark}
\label{rm:large}
The hyper-parameter $B$ plays an important role when controlling for $k$-FP. $T_k$ may be very conservative if $B = |\mathcal{Y}|$ and $|\mathcal{Y}|$ is very large, to the point where $T_k = -\infty$ always if $ |\mathcal{Y}| > k (n + 1)$. It can therefore be beneficial to truncate the considered label space $\mathcal{Y}$ for an example $x$ to only the top $B \ll k (n + 1)$ individual candidates, $\{y_1, \ldots, y_B\} \in \mathcal{Y}^B$. For example, for text generation tasks (like machine translation), $\mathcal{Y}$ is infinite, but we can restrict our predictions to a subset of the top $B$ beam search candidates (where $B$ can still be reasonably large). Still, this isn't free: a smaller $B$  may result in fewer true positives.
\end{remark}
\begin{remark}
No constraints are placed on the underlying set function $\mathcal{F}$ in Theorem~\ref{thm:calibration}; i.e., it need not be a $\mathrm{DeepSets}$ architecture. If, however, $\mathcal{F}$ is a good estimator of $\mathrm{FP}(Z, \mathcal{S}) ~|~ X$, then our method is more likely to identify sets that are approximately valid conditioned on $X_{n+1} = x_{n+1}$, which we investigate empirically in \S\ref{sec:results}. 
\end{remark}
\begin{remark}
Nestedness of $\mathcal{S}_{i,j}$ is not necessary for the above calibration to hold (it is mainly used for efficiency). The monotonicity of $\mathrm{FP_{max}}$ is sufficient.
\end{remark}

\subsection{Selecting the final output set}
\label{sec:greedy}
 The main consequence of Theorem~\ref{thm:calibration} is that, using the calibrated  nonconformity threshold $T_k$ or $T_{k, \delta} =  t^*$, we can construct a collection of sets that are simultaneously valid by keeping all candidate sets with scores less than $t^*$. Specifically, we are free to select \emph{any} set in the filtered set of candidates $S_{n+1, j}$, $j \in \mathcal{J}$ where $\mathcal{J}  := \{j \colon v_{n+1, j} <  t^*\}$, as a valid output.
Ideally, we would be able to follow the the oracle strategy in returning the smallest set with the highest number of true positives. This would make our predictions \emph{efficient}, in the sense that we are not including more false positives than necessary (even if the total is still $\leq k$). A reasonable choice is to then choose $\mathcal{S}_{j^*}$ where $j^* := \argmax_{j \in \mathcal{J}} |\mathcal{S}_j| - \mathcal{F}(x, \mathcal{S}_j)$; but this can be sub-optimal if $\mathcal{F}$ is not accurate.
As a greedy, but effective, approach we simply take the \emph{largest} set  as our final output, which has maximal TPR.   We formalize this in Proposition~\ref{prop:algo}.
\begin{proposition}[Greedy FP-CP]
\label{prop:algo}
Let $T_\circ$ denote either $T_k$ or $T_{k, \delta}$. Then random candidate sets $\mathcal{S}_{n+1, j}$,  $\forall j \in \mathcal{J}  := \{j \colon V_{n+1, j} <  T_\circ\}$, are valid. Furthermore, among indices $\mathcal{J}$, $\max \mathcal{J}$ indexes the nested set with the highest TPR.
\end{proposition}
We discuss some additional considerations of our method, as well as potential limitations and extensions, in Appendix~\ref{app:practical}.

%% file: sections/appendix/fp_control.tex
\definecolor{darkgreen}{rgb}{0.31, 0.47, 0.26}
% \section{Pseudo code}
% 
\begin{figure*}[t!]
    \label{app:algorithm}
    \centering
    \begin{minipage}{1\linewidth}
    \begin{algorithm}[H]
    \footnotesize
    \caption{ \small Pseudocode for conformal prediction with limited false positives (in expectation case, see Eq.~\eqref{eq:kfd}).}
    \label{alg:psuedocode}
    \textbf{Definitions:} \textcolor{black}{$x_{\mathrm{n+1}}$ is a test point, $\mathcal{D}_\mathrm{train}$ is a training set, $\mathcal{D}_\mathrm{cal}$ is a calibration set, $k$ is the tolerance, and $B$ is a  parameter for considering only the top  individually ranked candidates, $ y_c \in \mathcal{Y}$. $\mathrm{LikelihoodModel}$ is an abstract model that estimates individual label likelihood  for ranking and set item  featurization. $\mathrm{SetModel}$ is an abstract model that estimates FP (we use $\mathrm{DeepSets}$).}
    %$\mathrm{MPNN}$ is a property prediction model~\cite{chemprop}, $\mathrm{DeepSets}$ is a set prediction model~\cite{zaheer2017sets}.
    % $(\mathcal{S}_1, \ldots, \mathcal{S}_m)$ is a sequence of nonconformity measures. $\mathcal{M}$ is a monotonic correction controlling for family-wise error. $x_{n+1} \in \mathcal{X}$ is the given test point. $x_{1:n} \in \mathcal{X}^{n}$ and $y_{1:n} \in \mathcal{Y}^n$ are the previously observed calibration examples and their labels, respectively. $\mathcal{Y}$ is the label space. $\epsilon$ is the tolerance level. 
    \begin{algorithmic}[1]
    \STATE{\textcolor{darkgreen}{{\# Using a training set, fit both a  $\mathrm{LikelihoodModel}$ $p_\theta$ and a $\mathrm{SetModel}$ $\mathcal{F}$ (\S\ref{sec:deepsets}).}}}
    \FUNCTION{\textsc{train}($\mathcal{D}_{\mathrm{train}}$)}

        \STATE{$\mathcal{D}_{\mathrm{train}}^{(1)}, \mathcal{D}_{\mathrm{train}}^{(2)}$}{~$\leftarrow$~\textsc{split}($\mathcal{D}_{\mathrm{train}}$)}
        \hfill{\textcolor{darkgreen}{{\# Split the training data into two disjoint sets.}}}
       
        \STATE{$p_\theta(y_c \in Z \mid x)$}{~$\leftarrow$~\textsc{fit}($\mathrm{LikelihoodModel}, \mathcal{D}_{\mathrm{train}}^{(1)}$)} 
        \hfill{\textcolor{darkgreen}{{\# Use one set to estimate individual likelihoods, $p_\theta(y_c \in Z \mid x)$.}}}
        
        \STATE{${\mathcal{F}}(x, \mathcal{S})$}{~$\leftarrow$~\textsc{fit}($\mathrm{SetModel}, p_\theta, \mathcal{D}_{\mathrm{train}}^{(2)}$)} 
        \hfill\textcolor{darkgreen}{{\# Use the other (smaller) set to learn the $\mathrm{FP}$ set function, $\mathcal{F}(x, \mathcal{S})$.}}

        \STATE{\textbf{return}\ $p_\theta$, $\mathcal{F}$}
    \ENDFUNCTION
    \vspace{3pt}
    
    \STATE{\textcolor{darkgreen}{{\# Using the trained $p_\theta$ and $\mathcal{F}$ models, find a set score threshold $t_k$ on a calibration set that achieves $k$-FP validity (\S\ref{sec:calibration}).}}}
    \FUNCTION{\textsc{calibrate}($p_\theta$, $\mathcal{F}$, $\mathcal{D}_{\mathrm{cal}}$, $k$, $B$)}
        \STATE{$ \mathcal{T}_{\mathrm{cal}} = \{\}$}
         \FOR{$(x_i, z_i) \in \mathcal{D}_\mathrm{cal}$}
            \STATE {$\{y_{i,\pi(1)}, \ldots, y_{i,\pi(B)}\}$} {$\leftarrow  \textsc{sort}(\mathcal{Y}, p_{\theta}(y_c \in Z_i \mid x_i ))_{1:B}$\phantom{$\mathcal{D}^{(1)}$}}
            \hfill\textcolor{darkgreen}{{\# Rank top $B$ candidates by individual likelihood.}}
            %\vspace{1.5pt}
            \STATE {$\{\mathcal{S}_{i,1}, \ldots, \mathcal{S}_{i,B}\}$} {$\leftarrow  \{y_{i,\pi(1:j)} \colon j \in \{1, \ldots, B\}\}$\phantom{$\mathcal{D}^{(1)}$}}
            \hfill \textcolor{darkgreen}{{\# Construct nested sets using this ordering.}}
            
            \STATE {$\{v_{i,1}, \ldots, v_{i,B}\}$} {$\leftarrow  \{ \mathcal{F}(x_i, \mathcal{S}_{i,1}), \ldots \mathcal{F}(x_i, \mathcal{S}_{i,B})\}$\phantom{$\mathcal{D}^{(1)}$}}
            \hfill \textcolor{darkgreen}{{\# Compute nonconformity scores using $\mathcal{F}$.}}
            
            \STATE{$\mathrm{FP_{max}}(x_i, z_i, t)$}{$~\leftarrow~ \textsc{cache}(x_i, z_i, v_{i,1:B}, \mathcal{S}_{i,1:B})$\phantom{$\mathcal{D}^{(1)}$}}
            \hfill \textcolor{darkgreen}{{\# Cache dependent variables for $\mathrm{FP_{max}}(x_i, z_i, t)$.}}

            \STATE {$\mathcal{T}_{\mathrm{cal}}$}{$~\leftarrow~ \mathcal{T}_{\mathrm{cal}} \cup \{\mathrm{FP_{max}}(x_i, z_i, t)\}$\phantom{$\mathcal{D}^{(1)}$}}
            \hfill\textcolor{darkgreen}{{\# Append cached $\mathrm{FP_{max}}(x_i, z_i, t)$ to the calibration set.}}
        \ENDFOR
        
        \STATE{$t_k$}{$~\leftarrow~ \textsc{find\_threshold}(\mathcal{T}_{\mathrm{cal}}, B, k)$ \hfill\textcolor{darkgreen}{{\# Use Eq.~\eqref{eq:caltk} to find a $k$-FP valid set score threshold.}}}
        
        \STATE{\textbf{return}\ $t_k$}
    \ENDFUNCTION
    
    \vspace{3pt}
    \STATE{\textcolor{darkgreen}{{\# Using trained $p_\theta$ and $\mathcal{F}$ models and calibrated threshold $t_k$, return a TPR-maximizing prediction set for test point $x_{n+1}$ (\S\ref{sec:greedy}).}}}
    \FUNCTION{\textsc{predict}($x_{n+1}$, $p_\theta$, $\mathcal{F}$, $t_k$, $B$)}
        \STATE{\textcolor{darkgreen}{\# Repeat lines 12-14  to compute $\mathcal{S}_{n+1, 1:B}$ and ${v}_{n+1, 1:B}$.\phantom{$\mathcal{D}^{(1)}$)}}}

        \STATE {$\mathcal{J}$}{$~\leftarrow~ \{j \in \{1, \ldots, B\} \colon v_{n+1,j} <  t_k  \}$\phantom{$\mathcal{D}^{(1)}$}}
        \hfill \textcolor{darkgreen}{{\# Identify indices of candidate sets that pass threshold $t_k$.}}

        \STATE {$\cset(x_{n+1})$}{$~\leftarrow~ \mathcal{S}_{n+1, \max \mathcal{J}}$\phantom{$\mathcal{D}^{(1)}$}}
        \hfill \textcolor{darkgreen}{{\# Choose the largest sized set among filtered candidates.}}
        
        \STATE{\textbf{return}\ $\cset(x_{n+1})$}
    \ENDFUNCTION
    \end{algorithmic}
    \end{algorithm}
    \end{minipage}
    \vspace{-12pt}
\end{figure*}

%% file: sections/setup.tex
\section{Experimental setup}\label{sec:setup}

In this section, we outline our tasks and models.
We also describe our evaluation and baselines. 
For all experiments, we set  $B$ to $100$. Appendix~\ref{app:tasks} contains additional  details.
\subsection{Tasks} \label{sec:tasks}

\newpar{In-silico screening for drug discovery}
As introduced in \S\ref{sec:introduction},  the goal of in-silico screening is to identify potentially effective drugs to manufacture and test. We use the ChEMBL database~\cite{mayr} to screen molecules for combinatorial constraint satisfaction, where given a constraint such as ``\emph{has property A but not property B},'' we want to identify the subset of molecules from a given set of candidates that have the desired attributes. We partition the dataset both by molecules and property combinations, so that at test time the model makes predictions on combinations it has never been tested on before (after being trained on the same properties, but seen only in different combinations), over a pool of molecules that it has never seen before. Scores for candidate molecules are obtained via an ensemble of directed MPNNs~\cite{chemprop}.

\newpar{Object detection} We consider the task of placing bounding boxes around all objects of a certain type (such as a person) that are present in an image (of which there may be few, many, or none). We use the MS-COCO dataset~\cite{Lin2014MicrosoftCC}, a dataset with images of everyday scenes containing 80 object types (e.g., person, bicycle, dog, car, etc). We extract typed bounding box candidates (i.e., tuples of both location \emph{and} category) using an EfficientDet model~\cite{Tan2020EfficientDetSA} with non-maximum suppression. True positives are defined as  boxes that have an intersection over union (IoU) $>0.5$ with a matching annotation of the same type. 

\begin{figure*}[!t]
\small
\centering
% \hspace{-3mm}
\begin{subfigure}[b]{0.30\textwidth}
\includegraphics[width=1.05\linewidth]{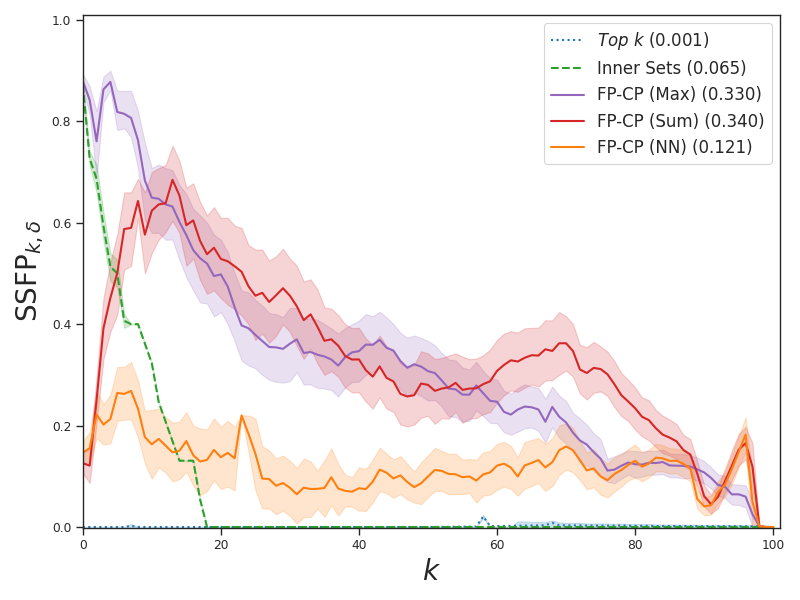} 
\vspace*{-1.1\baselineskip}
\end{subfigure}
~
% \hfill
\begin{subfigure}[b]{0.30\textwidth}
\includegraphics[width=1.05\linewidth]{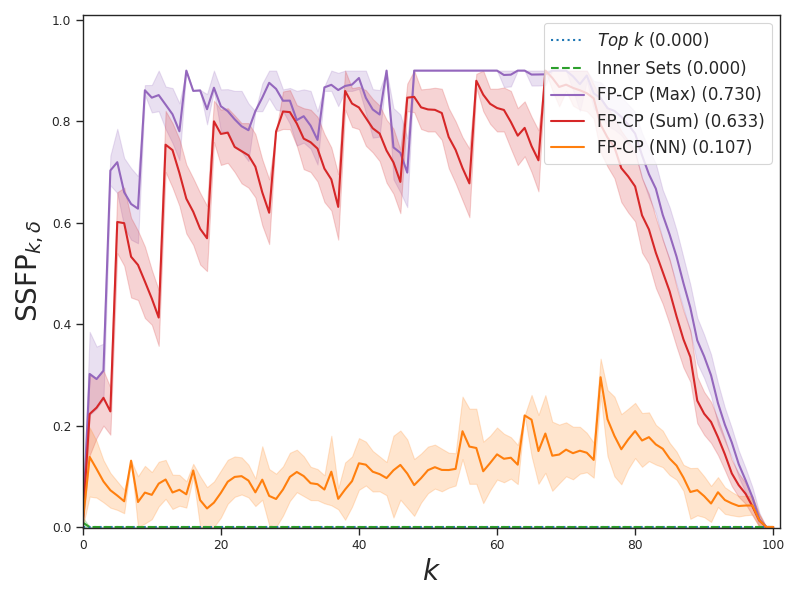}
\vspace*{-1.1\baselineskip}
\end{subfigure}
~
% \hfill
\begin{subfigure}[b]{0.30\textwidth}
\includegraphics[width=1.05\linewidth]{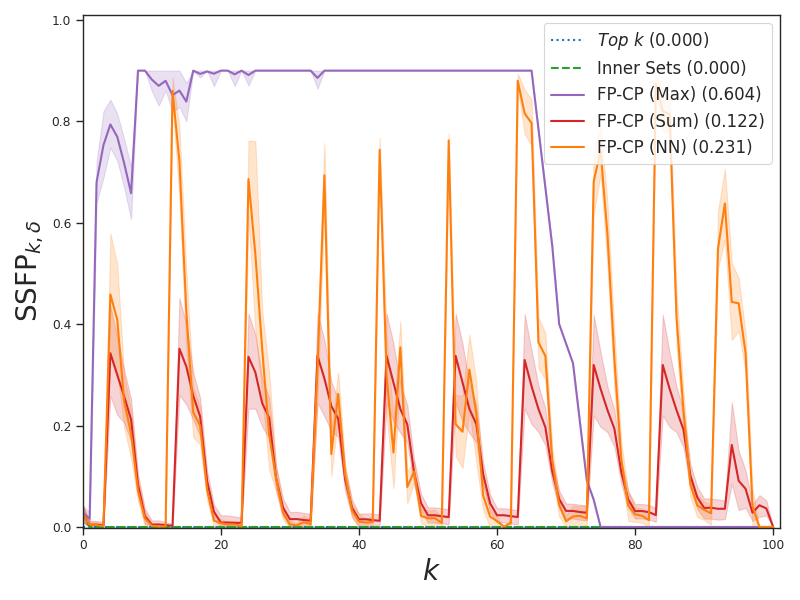} 
\vspace*{-1.1\baselineskip}
\end{subfigure}
\vspace*{-.85\baselineskip}
\begin{subfigure}[b]{0.30\textwidth}
\includegraphics[width=1.05\linewidth]{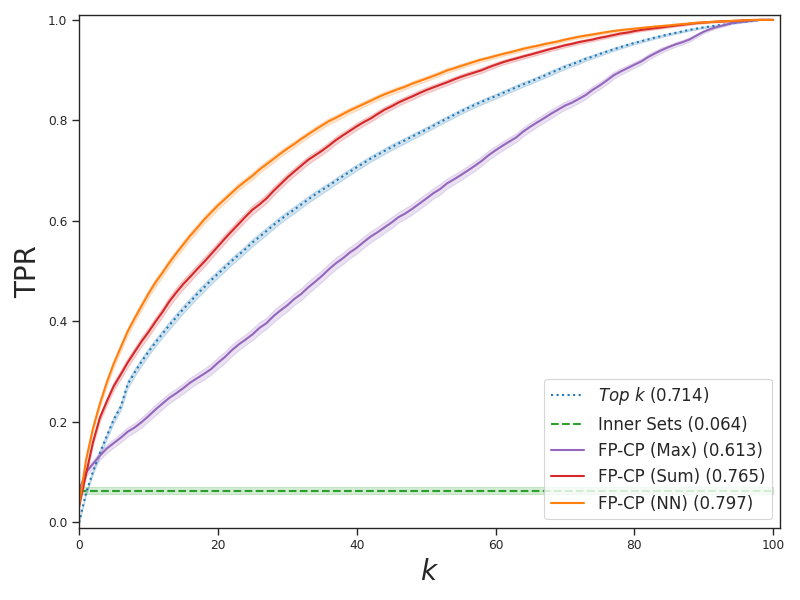} 
\vspace*{-1.8\baselineskip}
\caption{In-silico screening}
\end{subfigure}
~
% \hfill
\begin{subfigure}[b]{0.30\textwidth}
\includegraphics[width=1.05\linewidth]{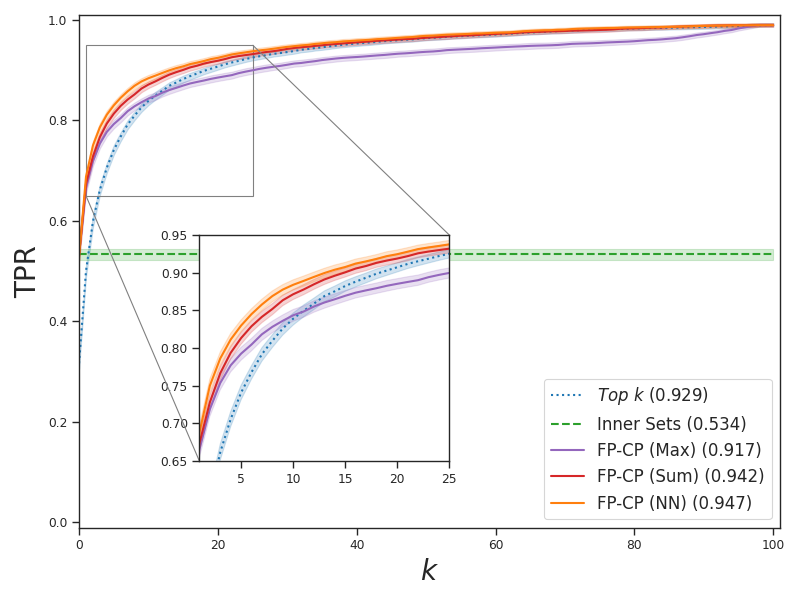}
\vspace*{-1.8\baselineskip}
\caption{Object detection}
\end{subfigure}
~
% \hfill
\begin{subfigure}[b]{0.30\textwidth}
\includegraphics[width=1.05\linewidth]{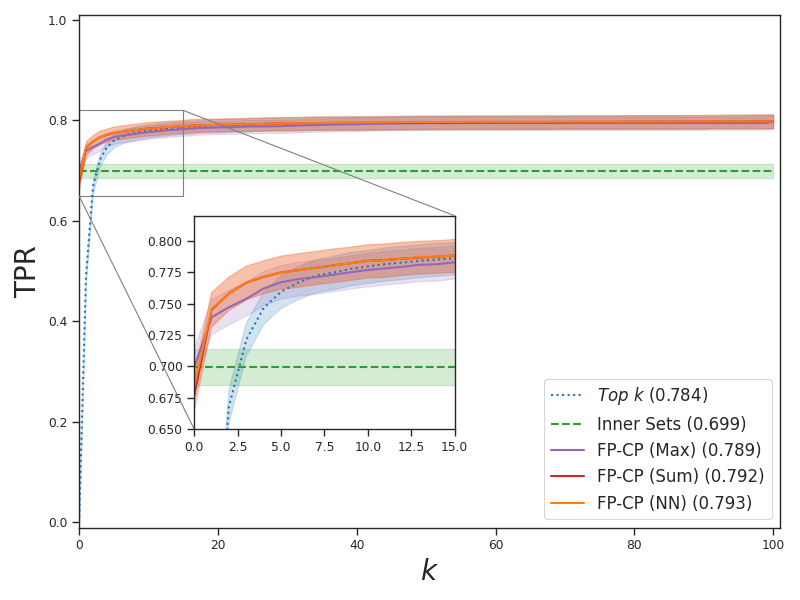} 
\vspace*{-1.8\baselineskip}
\caption{Entity extraction}
\end{subfigure}
%\vspace*{-0.5\baselineskip}
\caption{$(k, \delta)$-FP results as a function of $k$ for $\delta = 0.1$ up to $k = B = 100$. The top row plots $\mathrm{SSFP}_{k, \delta}$  violation (lower is better). The bottom row plots TPR (higher is better). We see that compared to the other baselines, our conformal $\mathrm{DeepSets}$ approach (NN) has the best (or close to) TPR AUC across tasks, while having the lowest (or close to) $\mathrm{SSFP}_{k, \delta}$ violation.}
\vspace{-15pt}
\label{fig:kdelta}
\end{figure*}

\newpar{Entity extraction} In entity extraction, we are interested in identifying all named entities that appear in a tokenized sentence $x$ of length $l$, where $x = \{w_1, \ldots, w_l\}$, and classifying them into appropriate categories. A named entity is a proper noun, demarcated by a contiguous span $\{w_{\mathrm{start}}, \ldots, w_\mathrm{end}\} \subseteq x$ of the input sentence, that can be associated with a particular class of interest (such as a person, location, organization, or product). We report results on the CoNLL NER dataset~\cite{tjong-kim-sang-de-meulder-2003-introduction}, where we use the PURE span-based entity extraction model of \citet{zhong-chen-2021-frustratingly} to individually score all $\mathcal{O}(l^2)$ candidate spans. We consider exact span predictions of the correct category to be true positives, and all others to be false positives. Many sentences contain no entities.

\subsection{Evaluation}
For each task we learn all models on a training set, perform model selection on a validation set, and report final
results as the average over 1000 random trials on a test set, where in each trial we partition the data into 80\% calibration ($x_{1:n}$) and 20\% prediction points ($x_{n+1}$).  
To compare across $k$, we plot each metric as a function of $k$ (up to $k = B$), and compute the area under the curve (AUC). Shaded  regions show the $16$-$84$th percentiles across trials. 
In addition to TPR (our main metric), as our method already guarantees marginal FP-validity, we also compute the size-stratified $k$-FP ($\mathrm{SSFP}_k$) and $(k, \delta)$-FP ($SSFP_{k, \delta}$) violation~\citep{angelopoulos2021sets}, see Appendix~\ref{app:ssfp}.
Lower size-stratified violation suggests that a classifier has better conditional coverage. We also report average FP results in Appendix~\ref{app:additionalexperiments}.

\begin{figure*}[!t]
\small
\centering
\begin{subfigure}[b]{0.30\textwidth}
\includegraphics[width=1.05\linewidth]{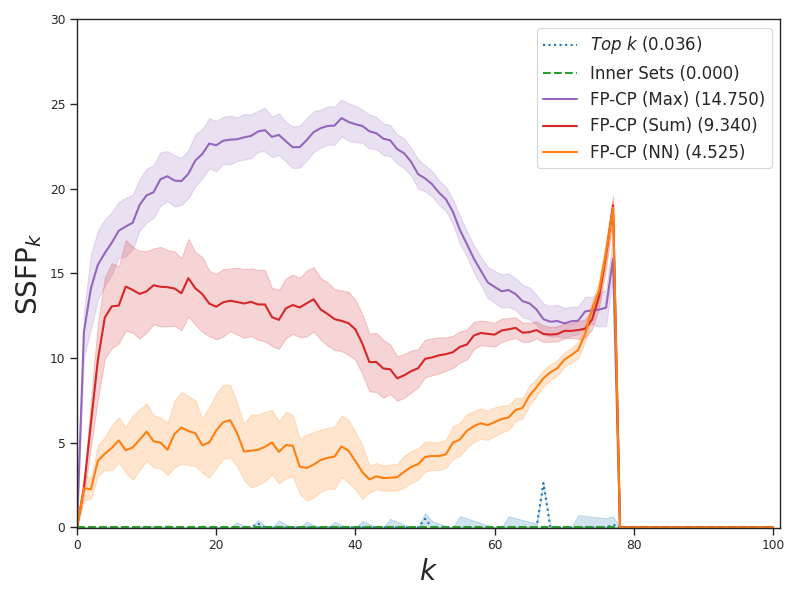} 
\vspace*{-1.1\baselineskip}
% \caption{QA}
\end{subfigure}
~
\begin{subfigure}[b]{0.30\textwidth}
\includegraphics[width=1.05\linewidth]{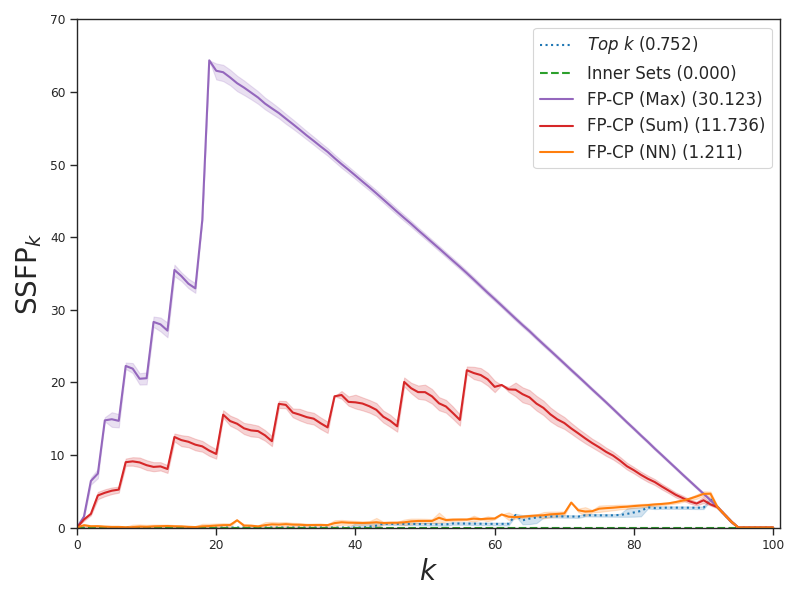}
\vspace*{-1.1\baselineskip}
% \caption{IR}
\end{subfigure}
~
\begin{subfigure}[b]{0.30\textwidth}
\includegraphics[width=1.05\linewidth]{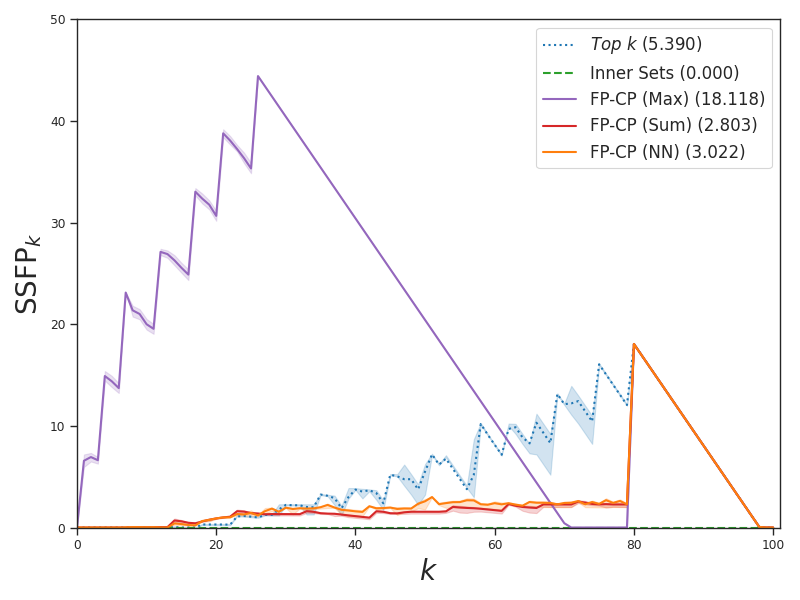} 
\vspace*{-1.1\baselineskip}
% \caption{DR}
\end{subfigure}
\vspace*{-.85\baselineskip}
\begin{subfigure}[b]{0.30\textwidth}
\includegraphics[width=1.05\linewidth]{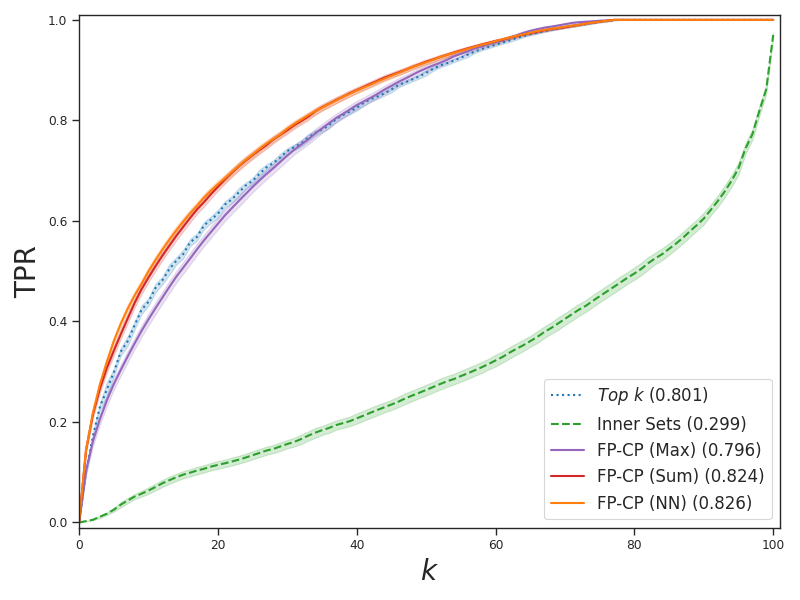} 
\vspace*{-1.8\baselineskip}
\caption{In-silico screening}
\end{subfigure}
~
\begin{subfigure}[b]{0.30\textwidth}
\includegraphics[width=1.05\linewidth]{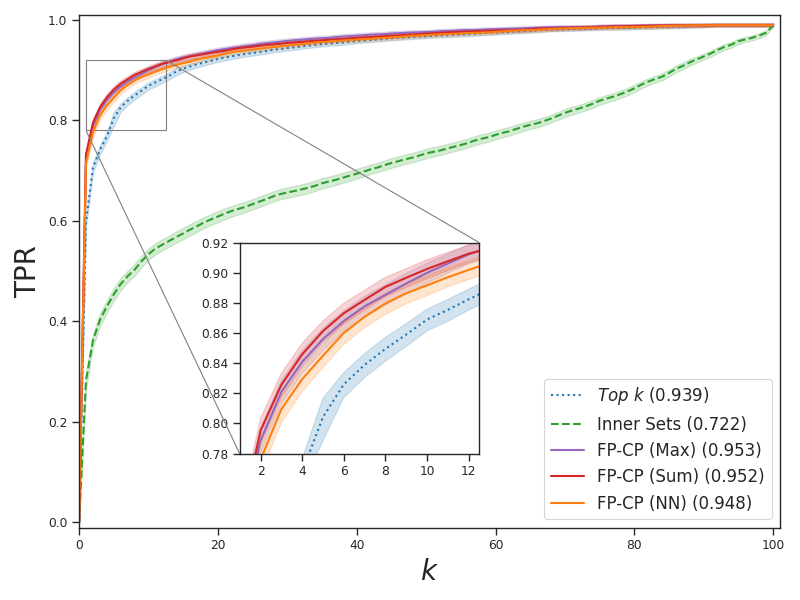}
\vspace*{-1.8\baselineskip}
\caption{Object detection}
\end{subfigure}
~
\begin{subfigure}[b]{0.30\textwidth}
\includegraphics[width=1.05\linewidth]{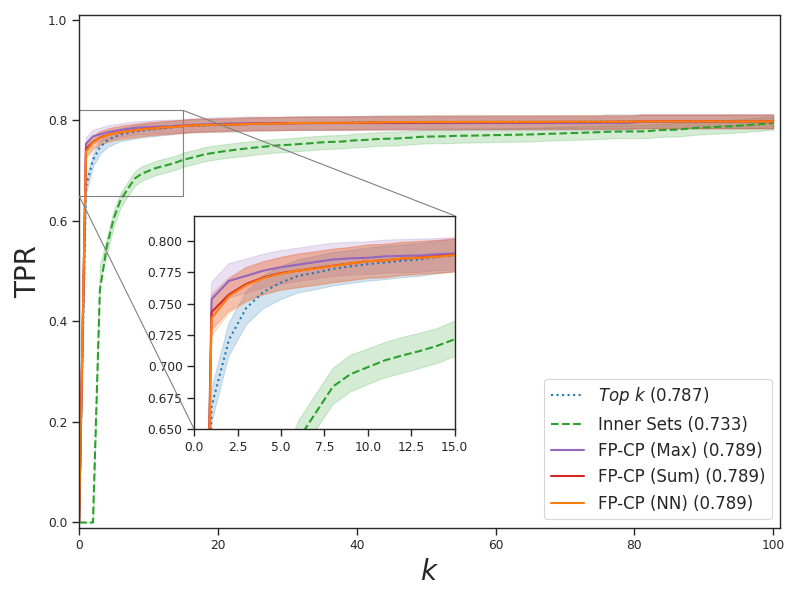} 
\vspace*{-1.8\baselineskip}
\caption{Entity extraction}
\end{subfigure}
%\vspace*{-0.5\baselineskip}
\caption{$k$-FP results as a function of $k$ up to $k = B = 100$. The top row plots size-stratified $k$-FP violation (lower is better). The bottom row plots the TPR (higher is better). As $k$ grows, our methods quickly achieve high TPR. Consistent with Figure~\ref{fig:kdelta},  our conformal $\mathrm{DeepSets}$ approach (NN) demonstrates high TPR and low $\mathrm{SSFP}_{k}$ across tasks.}
%\vspace{-18pt}
\vspace{-15pt}
\label{fig:k}
\end{figure*}

\subsection{Baselines}
\label{sec:baselines}
For all experiments, we compare our FP-CP (NN) method using a $\mathrm{DeepSets}$-based $\mathcal{F}$ to the following baselines:
%We compare to the following baselines:
\vspace{-5pt}
\begin{enumerate}[leftmargin=*, noitemsep]
    \item \textbf{Top-k.} We naively take the top $k'$ fixed predictions for any $x_{n+1}$, where $k'$ is found using average performance on the calibration set (without any correction factors, so it is \emph{not} guaranteed to be valid). Note that $k'$ can be (and mostly is) different than the user-specified $k$ for FP.\vspace{3pt}
    \item \textbf{Outer Sets @ 90.} \textcolor{black}{We use the (one-sided) multi-label conformal prediction technique of \citet{cauchois2020knowing} to bound $\mathbb{P}( Z_{n+1} \subseteq \mathcal{C}_\epsilon(X_{n+1}) ) \geq 0.90$. Though not directly comparable, we use this to benchmark our method against   sets that preserve marginal coverage (at a typical level). For simplicity, we use the direct inner/outer method without dynamic CQC quantiles.}\footnote{Preliminary experiments indicated that including CQC quantiles did not lead to significantly different (marginal) results.}
    \vspace{3pt}
    \item \textbf{Inner Sets.} Again, we use the (one-sided) method of \citet{cauchois2020knowing}, this time to bound $\mathbb{P}(\mathcal{C}_\epsilon(X_{n+1}) \subseteq Z_{n+1}) \geq 1 -\epsilon$ at level $\epsilon = k / 
    B$ (recall that $B \leq |\mathcal{Y}|$ is the truncation parameter, and the FP upper bound) for $k$-FP control and at level $\epsilon = \delta$ for $(k, \delta)$-FP control. It is straightforward to show that these levels of $\epsilon$ conservatively achieve $k$-FP and $(k, \delta)$-FP control. % For simplicity, we use a one-sided version of the direct inner/outer method described in \citet{cauchois2020knowing} without CQR quantiles.
    \vspace{3pt}
    %  (see Appendix~\ref{app:inner_outer} for details).
    \item \textbf{Independent scoring (max).} We take $\mathcal{F}(x, \mathcal{S})$ to be the maximum individual label uncertainty in $\mathcal{S}$, $\max\{1 - p_\theta(y_c \in Z \mid x) \colon y_c \in \mathcal{S}\}$. This is equivalent to choosing labels independently. Calibration uses the same FP-CP algorithm  (it is a drop-in replacement for the NN). 
    \vspace{3pt}
    \item \textbf{Cumulative scoring (sum).} We take $\mathcal{F}(x, \mathcal{S})$ to be the cumulative individual label uncertainty in $\mathcal{S}$, $\sum_{y_c \in \mathcal{S}} 1 - p_\theta(y_c \in Z \mid x)$. We calibrate $p_{\theta}(y_c \in Z \mid x)$ using Platt scaling~\cite{Platt99probabilisticoutputs}. As with the max scoring baseline, calibration uses the same FP-CP algorithm.
    \vspace{-5pt}
\end{enumerate}
Baseline (1) contrasts our approach with what is normally a ``first thought''  in practice, (2) and (3) test the efficacy of our system over existing techniques, and (4) and (5) demonstrate our FP-CP calibration with simpler alternatives for $\mathcal{F}$.

%% file: sections/results.tex
\section{Experimental Results}\label{sec:results}
In the following, we present our main findings.  Additional experimental results are included in Appendix~\ref{app:additionalexperiments}. 

 \newpar{Limiting false positives} 
 The top rows of Figures~\ref{fig:kdelta} and \ref{fig:k} show the size-stratified violation  for $(k,\delta)$-FP and $k$-FP, respectively. Across values of $k$, FP-CP (NN) typically achieves substantially \emph{lower} worst-case violations than either  max or sum scoring alternatives, (though, in some cases, the magnitude of $\mathrm{SSFP}$ can depend strongly on $k$). The Top-k and Inner Sets approaches also prevent large violations (though, by itself, this result is not necessarily impressive, as always returning an empty set will lead to $\mathrm{SSFP} = 0$). When accounting for TPR (bottom rows), we see that our FP-CP methods demonstrate stronger performance.
 
 \newpar{Maximizing true positive rates} 
The bottom rows of Figures~\ref{fig:kdelta} and \ref{fig:k} plot TPR rates and AUC across values of $k$, while Table~\ref{tab:chemres} details results for several representative individual configurations. On the screening task, we see that our FP-CP (NN) method provides significantly higher TPR than other baselines. For example, allowing no more than 5 false positives leads to a TPR of $36.1\%$ with $k$-FP. In comparison, the TPR of Top-k is only $29.8\%$. As might be expected, the advantage of the $\mathrm{DeepSets}$ approach underlying FP-CP (NN) over simpler FP-CP scoring mechansims is more pronounced for tasks with higher cardinality label sets, such as in-silico screening versus object detection of entity extraction (see a comparison of dataset characteristics in Table~\ref{tab:stats}). Furthermore, since entity extraction contains a high proportion of examples with ``empty'' label sets, we can see that its TPR asymptotes at the natural rate of answerable examples. Nevertheless, in general, all FP-CP methods (with max, sum, or NN scoring) provide high TPR (exceeding non FP-CP methods) even at low values of $k$.

\newpar{Comparison to conformal coverage methods}\textcolor{black}{ Table~\ref{tab:conformalcomparison} gives the results of the coverage-seeking Outer Sets method at level $0.90$ (a typical tolerance). Indeed, we achieve strong TPR ($97.2\%$ for the in-silico screening task), but also incur a high false positive cost in the process ($63.6$ average FP for in-silico screening). In contrast, our method allows us to directly limit false positives, without losing high TPR empirically (e.g., equivalently controlling for $\leq 63.6$ FP, we acheive $97.0\%$ TPR on the in-silico screening task).}
 
%  when using outer sets ($Z \subseteq \mathcal{C}$) at $\epsilon = 0.1$ we get
%  chem = (97.23 TPR, 63.61 FP, 86.55 size), coco=(96.09, 32.4 fp, 38.18), conll=(74.99, 0.7670 FP, 2.31)
% higher, but results for @k = 35 even competitive (and many more FP in the conformal)

%% file: sections/conclusion.tex
\section{Conclusion}
Conformal prediction, in its standard formulation, already grants theoretical performance guarantees that can be critical in many applications. Naively applying CP, however, can yield disappointing results. Even if the target coverage is upheld, the predicted sets may be too large, and too noisy, to be practical. In this paper, we proposed a method for trading coverage guarantees in favor of strict limits on the number of false positives contained in our prediction sets. Our results show that our method yields classifiers that (1) still achieve strong true positive rates compared to their coverage-seeking counterparts, and (2) predict meaningful output sets with effectively controlled numbers of false positives.

\section*{Acknowledgements}
We thank Ben Fisch, Anastasios Angelopoulous, Stephen Bates, and Lihua Lei for valuable technical feedback and discussions.
AF is supported in part by the NSF Graduate Research Fellowship. This work is also supported in part by MLPDS and the DARPA AMD project.

% \section*{Ethics Statement}
% Our FP control methods are general and can be applied to many applications and on top of any model for computing nonconformity scores. It's worth noting that any undesirable biases exhibited by underlying models can still propagate to the prediction sets of our methods. While our methods provide marginal performance guarantees, we recommend that any application to perform controlled evaluation across target populations to ensure fairness.

% \section*{Reproducibility Statement}
% All datasets used in this paper are publicly available (see \S\ref{sec:tasks}). Also, we use publicly available models for computing the nonconformity scores. For drug discovery, we use the ensemble of directed MPNNs from \texttt{chemprop} (\url{https://github.com/chemprop/chemprop}). 
% For object detection, we use \texttt{tf\_efficientdet\_d2} from \url{https://github.com/rwightman/efficientdet-pytorch}.
% For entity extraction, we train the PURE entity model (\url{https://github.com/princeton-nlp/PURE}) on the ConLL03 dataset using a context window of 100, batch size 32, \texttt{albert-base-v2} encoding model, and the suggested learning rate of $1e{-}5$, and task learning rate $5e{-4}$. Default values were used for all other parameters.
% The results in Section~\ref{sec:results} are all based on the experimental setting described in Section~\ref{sec:setup}. We will release our code for running these experiments and for reproducing all plots and tables.

%% file: sections/appendix/proofs.tex
\section{Mathematical details}
\label{app:proofs}

\subsection{Proof of Theorem~\ref{thm:conformalprediction}}
\begin{proof}
This is a well-known result~\cite{vovk2005algorithmic, Papadopoulos08, lei2018distribution, romano2019quantile}; we prove it here for completeness.
Since the nonconformity scores $V_i$ are constructed symmetrically, then
\begin{align*}
    &((X_1, Y_1), \ldots, (X_{n+1}, Y_{n+1})) \overset{d}{=} ((X_{\sigma(1)}, Y_{\sigma(1)}), \ldots, (X_{\sigma(n+1)}, Y_{\sigma(n+1)})) \\ &\Longleftrightarrow (V_1, \ldots, V_{n+1}) \overset{d}{=} (V_{\sigma(1)}, \ldots, V_{\sigma(n+1)})
\end{align*}
for all permutations $(\sigma(1), \ldots \sigma(n+1))$. Therefore, if $\{(X_i, Y_i)\}_{i=1}^{n+1}$ are exchangeable, then so too are their nonconformal scores $\{V_i = \mathcal{M}(X_i, Y_i)\}_{i=1}^{n+1}$ given exchangeability-preserving nonconformity measure $\mathcal{M}$.

 By the construction of $\mathcal{C}$, we have

\begin{equation*}
    Y_{n+1} \in \cset(X_{n+1}) \Longleftrightarrow V_{n+1} \leq \mathrm{Quantile}(1 - \epsilon, V_{1:n} \cup \{\infty\}).
\end{equation*}

This implies that $V_{n+1}$ is ranked among the $\lceil(1 - \epsilon)\cdot (n + 1)\rceil$ smallest of $V_1, \ldots, V_n, \infty$. Since $V_i$ are exchangeable, this happens with probability at least $1 - \epsilon$.
\end{proof}

\subsection{Proof of Lemma~\ref{lem:fpmax}}
\begin{proof}

We will not rely on nestedness of $\mathcal{S}_j$.

Notice that
\begin{equation}
    t \leq t' \Longrightarrow \{v_j \colon v_j < t\} \subseteq \{v_j \colon v_j < t'\}.
\end{equation}

As an immediate consequence,
\begin{align}
    t \leq t' &\Longrightarrow \{\mathrm{FP}(z, \mathcal{S}_j) \colon v_j \leq t\} \subseteq \{\mathrm{FP}(z, \mathcal{S}_j) \colon v_j \leq t\} \\
    &\Longrightarrow \max\{\mathrm{FP}(z, \mathcal{S}_j) \colon v_j < t\} \leq \max\{\mathrm{FP}(z, \mathcal{S}_j) \colon v_j < t'\} \\
    &\Longrightarrow \mathrm{FP_{max}}(x, z, t) \leq \mathrm{FP_{max}}(x, z, t').
\end{align}
\end{proof}
% For any $t' \geq t$, $\mathcal{V}(t) \subseteq \mathcal{V}(t')$.
% From its definition, it is obvious that $\mathrm{FP}(z, \mathcal{S}_j)$ is non-decreasing for $\mathcal{S}_j \subseteq \mathcal{S}_{j'}$,
% \begin{equation}
%     \mathcal{S}_j \subseteq \mathcal{S}_{j'} \Longrightarrow \mathrm{FP}(z, \mathcal{S}_j) \leq \mathrm{FP}(z, \mathcal{S}_{j'}).
% \end{equation}
% Furthermore, nestedness of sets $\mathcal{S}_j$ implies that
% \begin{equation}
%     j \leq j' \iff \mathcal{S}_j \subseteq \mathcal{S}_{j'},
% \end{equation}
% and therefore 
% \begin{equation}
%     j \leq j' \Longrightarrow \mathrm{FP}(z, \mathcal{S}_j) \leq \mathrm{FP}(z, \mathcal{S}_{j'}).
% \end{equation}

% From this we can conclude that $\mathrm{FP_{max}}(x, z, t)$ is simply  $\mathrm{FP}(z, \mathcal{S}_{j})$ for the largest $j$ satisfying $v_{j} \leq t$:
% \begin{equation}
%     \mathrm{FP_{max}}(x, z, t) = \mathrm{FP}(z, \mathcal{S}_{\hat{j}_t}), \quad \hat{j}_t := \max \{j \in \{1, \ldots, B\} \colon v_j \leq t\}.
% \end{equation}

% As a result,
% \begin{equation}
%     \hat{j}_t \leq \hat{j}_{t'} \Longrightarrow \mathrm{FP_{max}}(x, z, t) \leq \mathrm{FP_{max}}(x, z, t').
% \end{equation}

% Finally, from the definition of $\mathrm{FP_{max}}(x, z, t)$, we can conclude
% \begin{align}
%     t \leq t' &\Longrightarrow \max \{j \in \{1, \ldots, B\}, \colon v_j \leq t\} \leq \max \{j \in \{1, \ldots, B\}, \colon v_j \leq t'\} \\
%     &\Longrightarrow \hat{j}_t \leq \hat{j}_{t'} \\
%     &\Longrightarrow \mathrm{FP_{max}}(x, z, t) \leq \mathrm{FP_{max}}(x, z, t').
% \end{align}

\subsection{Proof of Theorem~\ref{thm:calibration}}

Our proof of Theorem~\ref{thm:calibration} builds on marginal RCPS \cite{arxiv-note}.  We restate their results here:
\begin{theorem}[Marginal RCPS]
\label{thm:marginallower}
Let $L_i \colon \mathbb{R} \rightarrow \mathbb{R}$, $i = 1, \ldots, n+1$ be exchangeable functions, where $L_i(t)$ is non-increasing in $t$. Also, take $g \colon \mathbb{R} \rightarrow \mathbb{R}$ where $g(x)$ is non-decreasing in $x$. Further assume that $g \circ L_i$ is right-continuous, and

\begin{equation}
    \inf_t g(L_i(t)) < \gamma, \quad\quad \sup_t g(L_i(t)) \leq B < \infty \text{ almost surely}.
\end{equation}

For any $\gamma \geq 0$, define the random variable $T(\gamma, g)$ as

\begin{equation}
    T(\gamma; g) := \inf \left\{t \colon \frac{1}{n+1} \sum_{i=1}^n g(L_i(t)) \leq \gamma \right\}.
\end{equation}

Then $\mathbb{E}[g \circ L_{n+1}(T(\gamma; g))] \leq \gamma + \frac{B}{n+1}$.
\end{theorem}
\begin{proof}
See \citet{arxiv-note}.
\end{proof}

We also restate Corollary 1 of \citet{arxiv-note}.

\begin{corollary}[Marginal RCPS, adjusted]
\label{cor:upper}
Under the same setting as in Theorem~\ref{thm:marginallower},
\begin{equation}
    \mathbb{E}[g \circ L_{n+1}(\widetilde{T}(\gamma; g))] \leq \gamma,
\end{equation}
where
\begin{equation}
    \widetilde{T}(\gamma; g) = \inf \left \{t \colon \frac{1}{n+1} \left( B + \sum_{i=1}^n g(L_i(t))\right)  \leq \gamma\right\}.
\end{equation}
\end{corollary}
\begin{proof}
See \citet{arxiv-note}.
\end{proof}

Following their analysis, we  provide an additional corollary for {lower}-bounding function $R_{n+1}$, where $R_1, \ldots, R_{n+1}$ are now non-{decreasing} exchangeable functions (as opposed to the non-{increasing}).

\begin{corollary}[Marginal RCPS, lower bound, non-decreasing case]
\label{cor:lower}
Similar to the setting in Theorem~\ref{thm:marginallower}, let $R_i \colon \mathbb{R} \rightarrow \mathbb{R}$, $i = 1, \ldots, n+1$ be exchangeable functions, where $R_i(t)$ is non-decreasing in $t$. Also, take $g \colon \mathbb{R} \rightarrow \mathbb{R}$ where $g(x)$ is non-decreasing in $x$. Further assume that $g \circ R_i$ is right-continuous, and 

\begin{equation}
    \inf_t g(R_i(t)) \geq  0, \quad\quad \sup_t g(R_i(t)) > C \geq \gamma  \text{ almost surely}.
\end{equation}

For any $\gamma \leq 0$, define the random variable $T(\gamma, g)$ as

\begin{equation}
    T(\gamma; g) := \inf \left\{t \colon \frac{1}{n+1} \sum_{i=1}^n g(R_i(t)) \geq \gamma \right\},
\end{equation}

where we define $\inf \varnothing = \infty$.
Then $\mathbb{E}[g \circ R_{n+1}(T(\gamma; g))] \geq \gamma$.

\end{corollary}
\begin{proof}
Let

\begin{equation}
    T'(\gamma; g) := \inf \left\{ t \colon \frac{1}{n+1} \sum_{i=1}^{n+1} g(R_i(t)) \geq \gamma \right\}.
\end{equation}

Since $\inf_t g(R_i(t)) \geq 0$, $\sup_t g(R_i(t)) > C \geq \gamma$, $T'(\gamma; g)$ and $T(\gamma, g)$ are both well-defined almost surely. 

Since $\inf_t g(R_i(t)) \geq 0$,

\begin{equation}
    \frac{1}{n+1}\sum_{i=1}^{n} g(R_i(t)) \geq \gamma \Longrightarrow  \frac{1}{n+1}\sum_{i=1}^{n+1} g(R_i(t)) \geq \gamma.
\end{equation}

Thus, $T'(\gamma; g) \leq T(\gamma; g)$. Since $g(R_i(t))$ is non-decreasing in $t$,

\begin{equation}
\label{eq:ineq}
    \mathbb{E}[g \circ R_{n+1}(T(\gamma; g))] \geq \mathbb{E}[g \circ R_{n+1}(T'(\gamma; g))].
\end{equation}

Let $E_f$ be the unordered set (bag) of $\{R_1, \ldots, R_{n+1}\}$. Then $T'(\gamma; g)$ is a function of $E_f$, and is a constant conditional on $E_f$. Exchangeability of $R_i$ and right-continuity of $g \circ R_i$ imply

\begin{equation}
    \mathbb{E}[g \circ R_{n+1}(T'(\gamma; g)) \mid E_f] = \frac{1}{n+1} \sum_{i=1}^{n+1} g \circ R_i(T'(\gamma; g)) \geq \gamma.
\end{equation}

As this is true given any $E_f$, we can take the expectation over $E_f$ to yield
\begin{equation}
    \mathbb{E}[\mathbb{E}[g \circ R_{n+1}(T'(\gamma; g)) \mid E_f]] \geq \gamma.
\end{equation}

The proof is completed by applying Eq.~\eqref{eq:ineq}.
\end{proof}

We now prove Theorem~\ref{thm:calibration}.

\begin{proof}

By Lemma~\ref{lem:fpmax}, we have that $\mathrm{FP_{max}}(x, z, t)$ is non-decreasing in $t$. It  is also easy to verify that $\mathrm{FP_{max}}(x, z, t)$ is left-continuous in $t$ and preserves exchangeability, so that $\mathrm{FP_{max}}(X_i, Z_i, t)$ are exchangeable functions of $t$. Next, define
\begin{equation}
\label{eq:barmFP}
   \mathrm{FP^-_{max}}(x, z, t) := \mathrm{FP_{max}}(x, z, -t),
\end{equation}
so that $\mathrm{FP^-_{max}}(x, z, t)$ is non-increasing in $t$ and right-continuous. Define random variables $T'_k$ and $T'_{k, \delta}$ as
\begin{align}
    T'_k &= \inf \Big\{t \in \mathbb{R} \colon \frac{B + \sum_{i=1}^{n}\mathrm{FP^{-}_{max}}(X_i, Z_i, t)}{n+1} \leq k\Big\}\quad\text{and}\\
    T'_{k,\delta} &= \inf \Big\{t \in \mathbb{R} \colon \frac{\sum_{i=1}^{n} \mathbf{1}\{\mathrm{FP^{-}_{max}}(X_i, Z_i, t) \leq k\}}{n+1} \geq 1 - \delta \Big\}.
\end{align}

We then have $T_k = -T'_k$ and $T_{k,\delta} = -T'_{k, \delta}$, which gives
\begin{align}
    \mathbb{E}\Big[\mathrm{FP}_{\mathrm{max}}(X_{n+1}, Z_{n+1}, T_k)\Big] &= \mathbb{E}\Big[\mathrm{FP}^-_{\mathrm{max}}(X_{n+1}, Z_{n+1}, T'_k)\Big]\quad\text{and} \label{eq:kequality}\\
 \mathbb{P}\Big(\mathrm{FP}_{\mathrm{max}}(X_{n+1}, Z_{n+1}, T_{k, \delta}) \leq k \Big) &= \mathbb{P}\Big(\mathrm{FP}^-_{\mathrm{max}}(X_{n+1}, Z_{n+1}, T'_{k, \delta}) \leq k \Big).\label{eq:kdequality}
\end{align}

(Part 1) We first prove $\mathbb{E}\big[\mathrm{FP^-_{max}}(X_{n+1}, Z_{n+1}, T'_k)\big] \leq k$. 

Since $B$ is finite, we have that $\sup_t \mathrm{FP^-_{max}}(x, z, t) \leq \max_j |\mathcal{S}_j| \leq B < \infty$. As we assume nonconformity scores are finite, we also have $\inf_t \mathrm{FP^-_{max}}(x, z, t) = 0 < k \in \mathbb{R}_{>0}$. Let $L_i(t) = \mathrm{FP^-_{max}}(X_i, Z_i, t)$ and $g(x) = x$. Corollary~\ref{cor:upper} gives

\begin{equation}
    \mathbb{E}\Big[\mathrm{FP}^-_{\mathrm{max}}(X_{n+1}, Z_{n+1}, T'_k)\Big] \leq k.
\end{equation}
Substituting Eq.~\ref{eq:kequality} gives $\mathbb{E}\big[\mathrm{FP_{max}}(X_{n+1}, Z_{n+1}, T_k)\big] \leq k$.

(Part 2) We now prove $\mathbb{P}\Big(\mathrm{FP}^-_{\mathrm{max}}(X_{n+1}, Z_{n+1}, T'_{k, \delta}) \leq k \Big) \geq 1 - \delta.$ 

Let $L_i(t) = \mathbf{1}\{\mathrm{FP^-_{max}}(X_i, Z_i, t) \leq k\}$. Let $g(x) = x$. As shown earlier, $\mathrm{FP^-_{max}}(X_i, Z_i, t)$ is non-increasing, right-continuous; as a result $L_i(t)$ is non-decreasing, right-continuous. Let $\gamma = 1 - \delta \in (0, 1)$. Since $g(L_i(t)) \in \{0, 1\}$ and $V_{i,j}$ are finite, it is easy to see that we have $\sup_t g(L_i(t)) = 1 \geq \gamma$ and $\inf_t g(L_i(t)) = 0 \geq 0$.

Applying Corollary~\ref{cor:lower} gives
\begin{align}
    \mathbb{E}\big[\mathbf{1}\{\mathrm{FP^-_{max}}(X_{n+1}, Z_{n+1}, T'_{k, \delta}) \leq k\}\big] &= \mathbb{P}\Big(\mathrm{FP}^-_{\mathrm{max}}(X_{n+1}, Z_{n+1}, T'_{k, \delta}) \leq k \Big) \\
    &\geq \gamma \\
    &= 1- \delta.
\end{align}
Substituting Eq.~\ref{eq:kdequality} gives $\mathbb{P}\Big(\mathrm{FP}_{\mathrm{max}}(X_{n+1}, Z_{n+1}, T_{k, \delta}) \leq k \Big) \geq 1 - \delta.$

\end{proof}

\subsection{Proof of Proposition~\ref{prop:algo}}
\begin{proof}
 We first prove simultaneous validity of candidate sets indexed by $j \in \mathcal{J}$. By definition we have
\begin{equation}
   \mathrm{FP}(Z_{n+1}, \mathcal{S}_j) \leq \mathrm{FP_{max}}(X_{n+1}, Z_{n+1}, T_{\circ}) \quad \forall j \in \mathcal{J},
\end{equation}
which implies
\begin{align}
    \mathbb{E}\Big[\mathrm{FP}(Z_{n+1}, \mathcal{S}_j) \Big] &\leq \mathbb{E}\Big[\mathrm{FP}_{\mathrm{max}}(X_{n+1}, Z_{n+1}, T_k)\Big]\quad \text{and} \\
 \mathbb{P}\Big(\mathrm{FP}(Z_{n+1}, \mathcal{S}_j) \leq k \Big) &\geq  \mathbb{P}\Big(\mathrm{FP}_{\mathrm{max}}(X_{n+1}, Z_{n+1}, T_{k, \delta}) \leq k \Big) 
\end{align}
simultaneously $\forall j \in \mathcal{J}$. Theorem~\ref{thm:calibration} then implies validity.

We now show maximal TPR (a simple outcome). If $\mathcal{S}_j \subseteq \mathcal{S}_{j'}$ then $y_c \in \mathcal{S}_j \Longrightarrow y_c \in \mathcal{S}_{j'}$ for any  $y_c \in z \subseteq \mathcal{Y}$.  Therefore

\begin{equation}
    \mathcal{S}_j \subseteq \mathcal{S}_{j'} \Longrightarrow \mathrm{TPP}(z, \mathcal{S}_j) \leq \mathrm{TPP}(z, \mathcal{S}_{j'}).
\end{equation}

Since candidate sets are nested,
\begin{equation}
    j \leq j' \Longrightarrow \mathcal{S}_j \subseteq \mathcal{S}_{j'},
\end{equation}
and
\begin{equation}
    \mathrm{TPP}(z, \mathcal{S}_{\max \mathcal{J}}) \geq \mathrm{TPP}(z, \mathcal{S}_{j'}) \quad \forall j' \in \mathcal{J}.
\end{equation}
Since this is true for all $(x, z)$, 
\begin{equation}
    \mathbb{E}\big[\mathrm{TPP}(Z_{n+1}, \mathcal{S}_{\max \mathcal{J}})\big] = \sup_{h \in \mathcal{H}}  \mathbb{E}\big[\mathrm{TPP}(Z_{n+1}, \mathcal{S}_{h \circ \mathcal{J}})\big]
\end{equation}
where $\mathcal{H}$ is the space of all possible index selection policies.
\end{proof}

%% file: sections/appendix/additional.tex
\section{Additional experimental details}

\subsection{Definition of size-stratified false positive violation}
\label{app:ssfp}

The size-stratified false positive (SSFP) violation measures the worst-case violation of our metric of interest (i.e., expectation or probability), conditioned on the \emph{size} of the output set $\mathcal{C}$. Specifically, $\mathrm{SSFP}_k$ and $\mathrm{SSFP}_{k, \delta}$ are defined as follows:

\begin{align}
    \mathrm{SSFP}_k(\mathcal{C}, \{\mathcal{A}\}_{s=1}^{a}) &:= \sup_s \max_{\phantom{j}}\bigg(\widehat{\mathbb{E}}\Big[\mathrm{FP}(Z_{n+1}, \cset(X_{n+1})) ~\Big|~ \{|\cset(X_{n+1})| \in \mathcal{A}_s\}\Big] - k, 0\bigg) \quad\text{and} \label{eq:ssFP}\\
    \mathrm{SSFP}_{k, \delta}(\mathcal{C}, \{\mathcal{A}\}_{s=1}^{a}) &:= \sup_s \max_{\phantom{j}}\bigg(\widehat{\mathbb{E}}\Big[\mathbf{1}\big\{\mathrm{FP}(Z_{n+1}, \cset(X_{n+1})) > k\big\} ~\Big|~ \{|\cset(X_{n+1})| \in \mathcal{A}_s\} \Big] - \delta, 0\bigg),
\label{eq:ssFPd}
\end{align}
where $\{\mathcal{A}\}_{s=1}^a$ forms a partition of $\{1, \ldots, |\mathcal{Y}|\}$, and  $\widehat{\mathbb{E}}$ denotes the empirical average over our test samples.

Following \citet{angelopoulos2021sets}, we show that if conditional validity holds for our objectives, then validity also holds after stratifying by set-size. Poor $\mathrm{SSFP}$ is therefore a symptom of poor conditional validity. 

In the following, we drop dependence on $n+1$ for clarity.

\begin{proposition}[Expectation case]
Suppose $\mathbb{E}[\mathrm{FP}(Z, \cset(X)) \mid X = x] \leq k$ for each $x \in \mathcal{X}$. Then,
\begin{equation}
    \mathbb{E}\Big[\mathrm{FP}(Z, \cset(X)) \mid \{|\cset(X)| \in \mathcal{A}\}\Big] \leq k,\quad\text{for any $\mathcal{A} \subset \{0, 1, 2, \ldots\}$}.
\end{equation}
\end{proposition}
\begin{proof}
\begin{align}
    \mathbb{E}[\mathrm{FP}(Z, \cset(X)) \mid \{|\cset(X)| \in \mathcal{A}\}]
        &= \frac{\mathbb{E}[\mathrm{FP}(Z, \cset(X)) \cdot \mathbf{1}{\{|\cset(X)| \in \mathcal{A}\}}]}{\mathbb{P}(|\cset(X)| \in \mathcal{A})} \\
        &= \frac{\mathbb{E}[\mathbb{E}[\mathrm{FP}(Z, \cset(X)) \cdot \mathbf{1}{\{|\cset(X)| \in \mathcal{A}\}} \mid X = x]]}{\mathbb{P}(|\cset(X)| \in \mathcal{A})} \\
        &= \frac{\int_x \mathbb{E}[\mathrm{FP}(Z, \cset(X)) \mid X = x]  \cdot \mathbf{1}{\{|\cset(x)| \in \mathcal{A}\}} d\mathbb{P}(x)}{\mathbb{P}(|\cset(X)| \in \mathcal{A})} \\
        &\leq \frac{ \int_x k \cdot \mathbf{1}{\{|\cset(x)| \in \mathcal{A}\}} d\mathbb{P}(x)}{\mathbb{P}(|\cset(X)| \in \mathcal{A})} \\
        & = k.
\end{align}
\end{proof}

\begin{proposition}[Probability case.] Suppose $\mathbb{P}( \mathrm{FP}(Z, \csetd(X)) \leq k \mid X = x) \geq 1 - \delta$ for each $x \in \mathcal{X}$. Then,
\begin{equation}
    \mathbb{P}\Big(\mathrm{FP}(Z, \csetd(X)) \leq k \mid \{|\csetd(X)| \in \mathcal{A}\}\Big) \geq 1 - \delta,\quad\text{for any $\mathcal{A} \subset \{0, 1, 2, \ldots\}$}.
\end{equation}
\end{proposition}
\begin{proof}
\begin{align}
\mathbb{P}(\mathrm{FP}(Z, \csetd(X)) \leq k \mid \{|\csetd(X)| \in \mathcal{A}\})
    &= \frac{\int_x \mathbb{P}( \mathrm{FP}(Z, \csetd(X)) \leq k \mid X = x) \cdot \mathbf{1}\{|\csetd(x)| \in \mathcal{A}\}d\mathbb{P}(x)}{\mathbb{P}(|\csetd(X)| \in \mathcal{A})} \\
    &\geq \frac{\int_x (1 - \delta) \cdot \mathbf{1}\{|\csetd(x)| \in \mathcal{A}\}d\mathbb{P}(x)}{\mathbb{P}(|\csetd(X)| \in \mathcal{A})} \\
    &= 1 - \delta.
\end{align}
\end{proof}

\subsection{Additional results}
Table~\ref{tab:chemres} reports the absolute results for a number of reference $k$ values on the in-silico screening task. For $k$-FP validity, we report the empirical average of false positives in the prediction sets. For $(k, \delta)$-FP validity we report the percentage of prediction sets with $\le k$ false positives.

We also include the experimental results of the Outer Sets @ 90 baseline (\S\ref{sec:baselines}) on the three evaluated tasks in Table~\ref{tab:conformalcomparison}.

\label{app:additionalexperiments}

\input{sections/table}

\begin{table}[!h]
\centering

\caption{Outer Sets results applied at coverage level $1 - \epsilon = 0.90$. Note that since some examples do \emph{not} have any positives, full coverage in the typical sense isn't always achievable or well-defined.}
\label{tab:conformalcomparison}

\begin{tabular}{@{}lccc@{}}
\toprule
Task & \multicolumn{1}{l}{TPR} & \multicolumn{1}{l}{Avg. FP} & \multicolumn{1}{l}{Avg. Size} \\ \midrule
In-silico screening & 97.2 & 63.6  & 86.6 \\
Object detection    & 96.1 & 32.4  & 38.2 \\
Entity extraction   & 75.0 & 0.77 & 2.31 \\ \bottomrule
\end{tabular}%

\end{table}

%% file: sections/table.tex
\begin{table*}[!h]
\centering
\small
\caption{Results for the in-silico screening task on the ChEMBL dataset. TPR and $\mathrm{FP} \leq k$ are expressed as percents. Our FP-CP methods meet our target thresholds; using the Inner Sets approach does too, but is conservative (as expected). Applying FP-CP calibration with our $\mathrm{DeepSets}$ model (NN) yields  substantially higher TPR across various tolerance levels.}
\label{tab:chemres}

%\resizebox{1\linewidth}{!}{%
\begin{tabular}{@{}lrr|rr|rr|rr|rr@{}}
\toprule
{\ul } &
  \multicolumn{2}{c}{\textbf{Top k}} &
  \multicolumn{2}{c}{\textbf{Inner Sets}} &
  \multicolumn{2}{c}{\textbf{FP-CP (Max)}} &
  \multicolumn{2}{c}{\textbf{FP-CP (Sum)}} &
  \multicolumn{2}{c}{\textbf{FP-CP (NN)}} \\
  \multicolumn{10}{l}{\textit{\underline{$k$-FP}}:} &  \\[1.2mm]
\textit{} &
  \textit{Avg. $\mathrm{FP}$} &
  \textit{TPR  } &
  \textit{Avg. $\mathrm{FP}$} &
  \textit{TPR  } &
  \textit{Avg. $\mathrm{FP}$} &
  \textit{TPR  } &
  \textit{Avg. $\mathrm{FP}$} &
  \textit{TPR  } &
  \textit{Avg. $\mathrm{FP}$} &
  \textit{TPR  } \\ \midrule
$k=5$    & 4.59  & 29.8 & 0.14  & 2.5  & 4.98  & 27.5 & 4.99  & 34.1 & 4.98  & {36.1} \\
$k=15$   & 14.47 & 53.4 & 0.88  & 9.5  & 14.98 & 50.7 & 14.99 & 58.8 & 14.99 & {59.9} \\
$k = 25$ & 24.51 & 68.0 & 1.49  & 13.4 & 24.98 & 66.8 & 24.99 & 73.1 & 24.99 & {73.2} \\
$k = 35$ & 34.54 & 78.2 & 2.45  & 18.4 & 34.97 & 78.4 & 34.99 & {82.6} & 34.99 & {82.5} \\ \midrule
\multicolumn{10}{l}{\textit{\underline{$(k, \delta)$-FP with $1-\delta = 0.9$}}:} &  \\[1.2mm]
\textit{} &
  \textit{$\mathrm{FP} \leq k$  } &
  \textit{TPR  } &
  \textit{$\mathrm{FP} \leq k$  } &
  \textit{TPR  } &
  \textit{$\mathrm{FP} \leq k$  } &
  \textit{TPR  } &
  \textit{$\mathrm{FP} \leq k$  } &
  \textit{TPR  } &
  \textit{$\mathrm{FP} \leq k$  } &
  \textit{TPR  } \\ \midrule
$k=5$    & 100.0 & 20.5 & 96.6  & 6.36 & 90.0  & 15.8 & 90.0  & 27.2 & 90.0  & {31.6} \\
$k=15$   & 94.7  & 42.4 & 99.5  & 6.36 & 90.0  & 26.7 & 90.0  & 47.4 & 90.0  & {55.3} \\
$k = 25$ & 96.6  & 55.7 & 100.0 & 6.36 & 90.0  & 37.4 & 90.0  & 62.3 & 90.0  & {69.0} \\
$k = 35$ & 97.5  & 66.2 & 100.0 & 6.36 & 90.0  & 49.1 & 90.0  & 74.0 & 90.0  & {79.0}\\
\bottomrule
\end{tabular}
%}%

\end{table*}

%% file: sections/appendix/tasks.tex
\section{Implementation and dataset details}
\label{app:tasks}

\newpar{In-silico screening} We construct a molecular property screening task using the ChEMBL dataset~\cite[see][]{mayr}. Given a specified constraint such as ``\emph{is active for property A but not property B},'' we want to retrieve at least one molecule from a given set of candidates that satisfies this constraint. The input for each molecule is its SMILES string, a notational format that specifies its molecular structure. The motivation of this task is to simulate in-silico screening for drug discovery, where it is often the case where chemists are searching for a new molecule that satisfies several constraints (such as toxicity and efficacy limits), out of a pool of many possible candidates.

We split the ChEMBL dataset into a 60-20-20 split of molecules, where 60\% of molecules are separated into a train set, 20\% into a validation set, and 20\% into a test set. Next, we take all properties that have at least 50 positive and negative examples (to avoid highly imbalanced properties). Of these properties, we take all N choose K combinations that have at least 100 molecules with all K properties labelled (ChEMBL has many missing values). We set K to 2. For each combination, we randomly sample an assignment for each property (i.e., $\{\text{active}, \text{inactive}\}^K$). We discard combinations for which more than 90\% of labeled molecules satisfy the constraint.
 We keep 5000 combinations for the test set, 767 for validation, and 4375 for training. The molecules for each of the combinations are only sourced from their respective splits (i.e., molecular candidates for constraints in the property combination validation split only come from the molecule validation split). Therefore, at inference time, given a combination we have never seen before, on a molecules we have never seen before, we must try to retrieve the molecules that have the desired combination assignment.
 
Our directed Message Passing Neural Network (MPNN) is implemented using the $\mathtt{chemprop}$ repository~\cite{chemprop}. The MPNN model uses graph convolutions to learn a deep molecular representation, that is shared across property predictions. Each property value (active/inactive) is predicted using an independent classifier head. The final prediction is based on an ensemble of 5 models, trained with different random seeds. Given a combination assignment $(Z_1 = z_1, \ldots, Z_k = z_k)$, we naively compute the joint likelihood independently, i.e., 
\begin{equation}
    p_\theta(Z_1 = z_1, \ldots, Z_k = z_k) = \prod p_{\theta}(Z_i = z_i).
\end{equation}

\newpar{Object detection} As discussed in \S\ref{sec:setup}, we use the MS-COCO dataset~\cite{Lin2014MicrosoftCC} to evaluate our conformal object detection. MS-COCO consists of images of complex everyday scenes containing 80 object categories (such as person, bicycle, dog, car, etc.), multiple of which may be contained in any given example. Since the official test set is hidden, we use the $5k$ examples from the development set and randomly partition them into sets of size 1$k$, 1$k$, and 3$k$ for calibration, validation, and testing, respectively. The EfficientDet model~\cite{Tan2020EfficientDetSA}\footnote{We use tf\_efficientdet\_d2 from \url{https://github.com/rwightman/efficientdet-pytorch}.} for extracting bounding boxes uses a pipeline of three neural networks to extract deep features, and then predict candidates. The model also uses a non-maximum suppression (NMS) post-processing step to reduce the total number of predictions by keeping only the one with the maximum score across highly overlapping prediction boxes. We merge the predictions of all classes into a unified set, where each element is a tuple of (class, bounding box). This means that multiple class predictions can be included for the same bounding box (i.e., there is class uncertainty), and multiple bounding boxes can be found for the same class (i.e., there are multiple objects in one image). We define true positives as predictions that have an intersection over union (IoU) value $>0.5$ with a gold bounding box annotation, \emph{and} that match the annotation's class. % Since the official test set is hidden, we use the $5k$ examples from the validation set and randomly partition them into $1/1/3k$ sized splits for training $\mathcal{F}$,  validation, and testing.
%That model is using a pipeline of three neural networks to extract deep features and predict candidates. We merge the predictions of all classes to a unified set, and define true positives as the predicted boxes with intersection over union (IoU) value $>0.5$ with a gold annotation for the predicted task.

% We evaluate our method on the object detection task using the MS COCO dataset~\cite{Lin2014MicrosoftCC}. Since the official test set is hidden, we use the $5k$ examples from the development set and randomly split them to 1/1/3$k$ for calibration, validation, and testing. We extract the box predictions of a EfficientDet model~\cite{Tan2020EfficientDetSA} after non-maximum suppression (NMS). That model is using a pipeline of three neural networks to extract deep features and predict candidates. We merge the predictions of all classes to a unified set, and define true positives as the predicted boxes with intersection over union (IoU) value $>0.5$ with a gold annotation for the predicted task.

\newpar{Entity extraction} Entity extraction is a popular task in natural language processing. Given a sentence, such as ``\emph{Barack Obama was born in Hawaii},'' the goal is to identify and classify all named entities that appear---i.e., (``Barack Obama'', Person) and (``Hawaii'', Location). We use the CoNLL NER dataset~\cite{tjong-kim-sang-de-meulder-2003-introduction}, and extract $1k$ examples for calibration out of the $3.3k$ development set, and report test results on the $3.5k$ test set. For our base model, we use the entity extraction module of PURE~\cite{zhong-chen-2021-frustratingly}, that predicts span scores with a classifier head on top of Albert-base~\cite{Lan2020ALBERT} contextual embeddings. The classification head has two non-linear layers and uses the learned contextual embeddings of the span start and end tokens, concatenated with a learned span width embedding. We train the model on the training set of the CoNLL NER dataset. We use the official code repository\footnote{\url{https://github.com/princeton-nlp/PURE}.} and the following parameters: $1e-5$ learning rate, $5e-4$ task learning rate, 32 train batch size, and 100 context window. Similar to our object detection task, we treat exact span predictions of the correct category as true positives, and any other entity predictions as false positives. As illustrated in Table~\ref{tab:stats}, a fairly large number of sentences do not contain any entities at all, while other sentences may contain several.

% For natural language processing, we use the CoNLL-2003 Named
% Entity Recognition (NER) dataset~\cite{tjong-kim-sang-de-meulder-2003-introduction}. For CoNLL, we extract $1k$ examples for calibration out of the $3.3k$ development set, and report test results on the $3.5k$ test set. We use the entity extraction module of PURE~\cite{zhong-chen-2021-frustratingly} that predicts span scores with a classifier on top of Albert-base~\cite{Lan2020ALBERT} contextual embeddings. We treat exact span predictions of the correct category as true positives, and any other entity predictions as false positives.

\begin{table}[!hbpt]
\centering
\caption{Dataset statistics (test split). Numbers are reported with respect to the top $B = 100$ candidates per example. The median number of positives and negatives per example is given, in addition to their $16$th and $84$th percentiles. We also give statistics for the percentage of examples that have ``empty'' label sets with no positives (i.e., the label set has $|z| = 0$).}
\label{tab:stats}

\begin{tabular}{@{}llcccc@{}}
\toprule
Dataset & Input & \multicolumn{1}{l}{\# Examples} & \multicolumn{1}{l}{\# Negatives} & \multicolumn{1}{l}{\# Positives} & \multicolumn{1}{l}{\% Empty} \\ \midrule
In-silico screening & SMILES     & 5,000 & 85 (50-97)   & 15 (3-50) & 0.0  \\
Object detection    & Image            & 3,000 & 96  (89-98)  & 4 (2-11)  & 1.1  \\
Entity extraction   & Text & 3,453 & 99  (97-100) & 1 (0-3) & 20.2 \\ \bottomrule
\end{tabular}%

\end{table}

\newpar{Code} Our code will be available soon at \url{https://github.com/ajfisch/conformal-fp}.

%% file: sections/appendix/practical.tex
\section{Practical considerations and limitations}
\label{app:practical}

In this section we address a number of practical considerations, limitations, and extensions for our FP-CP method.

\subsection{Choosing a suitable k}
\label{sec:choosek}
An outstanding question a practitioner faces is how to choose the value of $k$ for $k$-FP and $(k, \delta)$-FP objectives. The value of $k$ in our method has a reliable and easy interpretation: it is the total number of incorrect answers. For many tasks, such as in-silico screening, there is a direct relation between the number of noisy predictions (e.g., failed experiments conducted during wet-lab validation) and total ``wasted'' cost. Therefore, for example, given some approximate budget $Q$ and cost per noisy prediction $c$, a reasonable approach is to then set $k \approx Q / c$. Of course, the advantage of our approach is that the user may set $k$ to whatever they wish---this might change based on their needs, and is not part of our algorithm.

\subsection{Choosing between k-fwer and fdr control}

A related question to \ref{sec:choosek} is when to target $k$-FWER (i.e., our $k$-FP and $(k, \delta)$-FP objectives) or FDR (e.g., using \citet{anastasios-learning-2021}). This choice is well discussed in the multiple testing literature~\cite{lehmann-fwer-2005, romano-control-2007, 10.1093/bioinformatics/btp385}. An important aspect to consider is the size of the label space $\mathcal{Y}$, natural rate of true and false positives, and the efficiency of the base model at separating true positives from false positives. When the total number of true positives  is large and $|\mathcal{Y}|$ is large then it is reasonable to control the FDR. If, however, the natural rate of true positives is low, or they are not well separated from false positives, then the FDR can be high and hard to control. This is especially true for smaller prediction sets (as the ratio of positive to negative labels can be quickly driven down even with the addition of only a few false positives). For illustration, suppose for a given example there is one true positive that is ranked 10th by the base model. For many applications, 10 total predictions (with 9 false positives) is acceptable. Yet, the lowest FDR cutoff that allows for this positive to be discoverable is $0.9$ (which, for other examples, may allow for hundreds of false positives---an outcome which is undesirable for some applications, even given a high number of accompanying true positives). To satisfy a lower FDR, the algorithm must output an empty set (with FDR = 0). This remains true even if there are a few (but not many) other true positives: for instance, in the previous example, if predictions 10-20 were also all true positives then the lowest FDR is still only 0.5---specifying a FDR tolerances any lower than this would force an  empty set prediction.

\subsection{Learning more expressive set functions}

Our choice of $\mathrm{DeepSets}$ model is motivated by its property of being a universal approximator for continuous set functions, and by its efficient architecture. Of course, its realized accuracy depends on its exact parameterization and optimization. In terms of input features, in \S\ref{sec:deepsets}, we chose a simplistic $\phi(x, y_c)$ for two reasons: (1) we view it’s low complexity as an advantage (practitioners can simply plug-in individual multi-label probabilities, or other scalar conformity scores, that most out-of-the-box methods provide into a general framework without having to do any more work for providing additional features), and (2) it is easy to train this light-weight model on smaller amounts of data. Still, this approach can discard potentially helpful information about the input $x$, and any dependencies between labels $y_c$ and $y_c'$. For example, if $y_c$ and $y_c'$ are mutually exclusive, then the number of false positives if both are included in $\mathcal{S}$ is at least $1$. Using more expressive $\phi$ that is able to capture and take advantage of this sort of side information about $x$ and $y_c$ is a subject for future work.

\subsection{Constructing non-nested candidate sets}

We choose to construct nested prediction sets because they are efficient and effective. It is useful to emphasize, however, that nestedness is not necessary for our calibration framework: our procedure still works even when candidate sets are not nested. It only relies on $\mathrm{FP_{max}}$ remaining monotonic in $t$, which is preserved even for non-nested candidate sets.  That said, generally speaking, considering individual candidates in the order of individual likelihood is a good strategy: this maximizes the expected number of true positives in a set of fixed size. Of course, we are not ranking by the true marginal likelihood, but rather the estimate, $p_{\theta}(y_c \in Z \mid x)$, and this may introduce some error. In theory, the set function $\mathcal{F}$ may be able to identify higher quality outputs sets $\mathcal{S} \in 2^\mathcal{Y}$ by jointly considering all of the included elements (rather than ranking them one-by-one). That said, an unconstrained search process over $2^\mathcal{Y}$ is expensive. Furthermore, identifying the final output set with maximal TPR, as we show we do in Proposition~\ref{prop:algo}, is no longer trivial. Nevertheless, this is a promising area for future work, can potentially be combined with efficient search or pruning methods (e.g., such as in \citet{fisch2021admission}).